%% file: weakBar_arxiv_final_version2.tex
\documentclass{article}

     \PassOptionsToPackage{numbers, compress}{natbib}

    \usepackage[preprint]{neurips_2021}

\usepackage[utf8]{inputenc} 
\usepackage[T1]{fontenc}    
\usepackage{hyperref}       
\usepackage{url}            
\usepackage{booktabs}       
\usepackage{amsfonts}       
\usepackage{nicefrac}       
\usepackage{microtype}      
\usepackage{xcolor}         

\usepackage[ruled]{algorithm2e}
\usepackage{float}
\usepackage{wrapfig}

\input{preamble}

\title{A novel notion of barycenter for probability distributions based on  optimal weak mass transport}

\author{%
  Elsa Cazelles\\
  IRIT, Université de Toulouse, CNRS\\
  \texttt{elsa.cazelles@irit.fr} \\
   \And
  Felipe Tobar \\
 IDIA \& CMM, Universidad de Chile\\
  \texttt{ftobar@dim.uchile.cl} \\
   \AND
Joaquin Fontbona\\
  CMM, Universidad de Chile\\
   \texttt{fontbona@dim.uchile.cl} \\
}

\begin{document}

\CB{
\vspace*{\fill}
\begin{center}
{\Large To the reader's attention}
\end{center}

The previous arXiv version of this paper  and the published version  contained an erroneous predicate : an optimal plan solving the weak OT problem  \eqref{def:weak_transport} might not be unique. However, there is  a  $\td\mu$-almost surely  uniquely defined mapping $S_{\mu}^{\nu}$ such that  the optimal law $\eta^*\leq_c\nu$ verifies $S_{\mu}^{\nu}\#\mu = \eta^*$ and  the barycentric projection $x\mapsto \int y \td\pi_x(y)$  of any optimal weak plan $\pi\in\Pi(\mu,\nu)$ is  $\td\mu$-almost surely equal to $S_{\mu}^{\nu}(x)$ (see Theorem \ref{th:1.4Julio}). As a main consequence, the statement and proof of  Lemma \ref{lemma:Smeasurable} regarding the joint measurability of the mapping $(x,\nu)\mapsto S_{\mu}^{\nu}(x)$ for fixed $\mu$ have been modified. The overall results of the paper are not affected by this change.

\vspace*{\fill}
}

\pagebreak

\maketitle

\input{chapters/CTF2021_abstract_final_version2}

\input{chapters/CTF2021_introduction_final_version2}

\input{chapters/CTF2021_background_final_version2}

\input{chapters/CTF2021_section_weakBarycenter_final_version2}

\input{chapters/CTF2021_section_weakPopBarycenter_final_version2}

\input{chapters/CTF2021_algorithms_final_version2}

\input{chapters/CTF2021_experiments_final_version2}

\section{Discussion}
\label{sec:conclusion}

We have introduced the weak barycenter, which extracts common geometric information of probability measures on $\R^d$ based on optimal weak transport, and showed that it can be interpreted as a latent variable model. From the fixed-point formulation defined in terms of optimal weak transport maps, irrespective of the regularity assumptions on the measures involved, we developed practical computation via an iterative algorithm with guaranteed convergence. In particular, the proposed algorithms do not require a common grid on the sample space, when processing either observed data  or samples from distributions.
We have also proposed weak barycenters of a possibly infinite population of measures and developed a stochastic procedure for computing it in the streaming data regime where distributions are processes into the weak barycenter as they arrive. This has critical implications for continual-learning methods in the ML community.

Additional studies will focus on deepen the latent variable interpretation of weak barycenters, and its relationship to the aggregate information represented by the Wasserstein barycenter.  Also, we identify two relevant theoretical aspects for further research: i) to exhibit general conditions on the family of input measures (or on the law of the population) for the existence of weak barycenters that are not Dirac masses; and ii) to provide conditions on those input measures  for a "maximal''  weak barycenter (in terms of convex ordering) to exist when $d\geq 2$,  among all the solutions of the weak barycenter problem (and, if possible, a way of constructing it by regularisation most probably).  The statistical behaviour of the weak barycenter can also  be  investigated, in particular when constructed from large empirical random samples of given distributions. Lastly, the weak population barycenter could also be  used to construct a predictive posterior in the context of Bayesian learning, as was done for Wasserstein barycenters in \cite{rios1805bayesian}.


\textbf{Acknowledgments.} We thank Julio Backhoff-Veraguas for his valuable insight during the writing of this paper. This work was funded by ANID grants: AFB170001 \& ACE210010 (CMM), FB0008 (AC3E), Fondecyt-Postdoctorado \#3190926 and Fondecyt-Iniciación \#1210606.

\appendix

\input{chapters/CTF2021_appendix_final_version2}

\end{document}

%% file: preamble.tex
\usepackage[utf8]{inputenc} 
\usepackage[T1]{fontenc}    
\usepackage{url}            
\usepackage{booktabs}       
\usepackage{amsfonts}       
\usepackage{nicefrac}       
\usepackage{microtype}      
\usepackage[ruled]{algorithm2e}
\usepackage{wrapfig}

\usepackage{graphicx}

\usepackage{amsbsy}
\usepackage{amssymb}
\usepackage{amsmath}
\usepackage{amsthm}
\usepackage{color}
\usepackage{tabularx}
\usepackage{booktabs}
\usepackage{bbm}

\usepackage{tikz,pgf,tikz-3dplot}
\usepackage{tikz}

\newcolumntype{C}[1]{>{\centering\arraybackslash}m{#1}}

\def\td{{\text d}}
\def\id{{\text{id}}}

\def\N{{\mathbb {N}}}

\def\d{{\text d}}

\def\L{{\mathbb L}}

\def\R{{\mathbb R}}

\def\N{{\mathbb N}}
\def\P{{\mathbb P}}

\def\E{ \mathbb{E} }

\def\T{ \mathbb{T} }
\def\Q{ \mathbb{Q} }

\def\b{{\mathbf b}}

\def\bfnu{{\boldsymbol \nu}}

\def\FF{ \mathcal{F} }

\def\PP{ \mathcal{P} }
\def\BB{ \mathcal{B} }

\def\NN{ \mathcal{N} }

\def\1{\mathbbm{1}}

\def\b0{ \mathbf{0} }

\DeclareMathOperator*{\argmin}{arg\,min}

\newtheorem{proposition}{Proposition}  
\newtheorem{definition}{Definition}
\newtheorem{theorem}{Theorem}
\newtheorem{remark}{Remark}
\newtheorem{lemma}{Lemma}

\newcommand{\CB}[1]{{\color{black} #1}}
\definecolor{salmon}{rgb}{1.0, 0.55, 0.41}
\definecolor{dodgerblue}{rgb}{0.12, 0.56, 1.0}
\definecolor{orchid}{rgb}{0.85, 0.44, 0.84}
\definecolor{limegreen}{rgb}{0.2, 0.8, 0.2}

\usepackage{caption}
\DeclareCaptionLabelFormat{andtable}{#1~#2  \&  \tablename~\thetable}

%% file: chapters/CTF2021_abstract_final_version2.tex

\begin{abstract}

We introduce weak barycenters of a family of probability distributions, based on the recently developed notion of optimal weak transport of mass \cite{gozlan2017kantorovich}, \cite{backhoff2020weak}. We provide a theoretical analysis of this object and discuss its interpretation in the light of convex ordering between probability measures. In particular, we show that, rather than averaging the input distributions in a geometric way (as the Wasserstein barycenter based on classic optimal transport does)  weak barycenters extract common geometric information shared by all the input distributions, encoded as a latent random variable that underlies all of them.  We also provide an iterative algorithm to compute a weak barycenter for a finite family of input distributions, and a stochastic algorithm that computes them for arbitrary populations of laws.  The latter approach  is  particularly well suited for the \emph{streaming setting}, i.e., when distributions are observed sequentially. 
The notion of weak barycenter and our approaches to compute it are illustrated on synthetic examples, validated on 2D real-world data and compared to standard Wasserstein barycenters.

\end{abstract}

%% file: chapters/CTF2021_introduction_final_version2.tex

\section{Introduction}

Optimal transport (OT) \cite{villani2003topics} has had a tremendous impact in the machine learning (ML) community recently, as it provides meaningful and implementable distances between probability distributions \cite{peyre2019computational}, thus advancing many aspects in the field, see e.g.  \cite{arjovsky2017wasserstein, ye2017fast, nguyen2013convergence}. The space of probability measures  on $\R^d$ with finite second moment can be \emph{metrised} with the Wasserstein-$2$ distance, the computation of which amounts to finding a transport plan that minimises the quadratic average cost of transporting mass from a source probability measure onto a target one.  
In this context, a natural method for averaging a finite family of probability measures is to compute their Fréchet mean, with respect to the Wasserstein-$2$ distance, which corresponds to the  {\it Wasserstein barycenter }introduced in \cite{agueh2011barycenters}. 

The goal of the present work is to explore theoretical features and  potential applications to ML  of barycenters of probability measures analogously defined in terms of {\it optimal  weak transport}  (OWT, see \cite{gozlan2017kantorovich}) or more precisely  quadratic  barycentric  transport costs. 
In a nutshell, for a source measure $\mu$ and a target measure $\nu$, the OWT problem aims to transport mass so that  the conditional spatial mean of target support points $y$, given their source support points $x$,  is close to $x$ in average. This amounts  to finding an intermediate measure $\eta$,  possibly  \emph{more concentrated} than $\nu$ in the sense of convex ordering of probability measures, which is \emph{close} to $\mu$ with respect to the  Wasserstein-$2$ distance. 
 
The main motivation of our work  is to investigate the effect and meaning of combining a  family of probability  measures using OWT instead of OT. To that end, we will define the   {\it weak  barycenter} of this family through an optimisation problem, and  discuss some of its properties. Importantly, we will see that,  rather than averaging the input distributions in a metric sense, solving a weak barycenter problem corresponds to finding probability measures that encode geometric or shape information \emph{shared across} all of them. In fact,  the weak barycenter  problem will be interpreted  as finding a latent random variable common to  all the input distributions. Implications of this latent variable interpretation, in terms of robustness to outliers,  will also be drawn in our work.

\CB{A second motivation for our work is to develop and implement computational methods for weak barycenters, capitalising on the fact that the barycentric projection \eqref{def:proj_bary} of {\it any} optimal weak coupling between a pair of distributions $(\mu,\nu)$ with finite second moments, uniquely defines a mapping pushing forward $\mu$ to the closest measure $\eta$ convexly smaller than $\nu$. This property contrasts with standard OT, where the absolute continuity with respect to the Lebesgue measure of  the source or target measure is typically needed to grant the existence and uniqueness of  a map---the so-called Monge map---  realising the optimal coupling between $\mu$ and $\nu$. This map is often required in different ways to compute Wasserstein barycenters (see \cite{alvarez2016fixed},  \cite{zemel2019frechet}  or  \cite{rios1805bayesian}). }


Similarly to the Wasserstein barycenter problem, we will develop a fixed-point formulation of the weak barycenter problem, based on OWT plans. This allows us to, following \cite{alvarez2016fixed, zemel2019frechet}, construct an iterative procedure to compute a  weak barycenter for a finite family of distributions and analyse its convergence properties. We will also define and study the so-called {\it weak population barycenters}, that deal with a population of probability measures distributed according to a given law $\Q$ supported on the Wasserstein-$2$ space, as in \cite{le2017existence} for the OT case. Extending ideas from \cite{bakchoff2022}, we will then propose an iterative stochastic algorithm for online computation of the weak population barycenter, from a \textit{stream} of probability measures sampled from $\Q$. We will then provide numerical simulations using this proposed method, in order to illustrate the geometric meaning of the weak barycenter, and we will compare it with related objects obtained with standard OT or its entropy-regularised counterpart.

\textbf{Organisation of the paper.} Sec.~\ref{sec:background} documents the background on OWT and the assumptions underlying our work. Sec.~\ref{sec:weak_bar} analyses the weak barycenter problem, interprets it in the light of convex ordering and a latent variable model and addresses the case of an infinite population of distributions. Sec.~\ref{sec:algorithms} introduces two algorithms for computing the weak barycenter in the finite or population settings.  Sec.~\ref{sec:computation} and \ref{sec:expe} present the experimental setting and validation of our proposal respectively. Lastly, Sec.~\ref{sec:conclusion} discusses our findings and future research questions. The Appendix contains all the proofs, additional details or our simulations and the code of our experiments.

%% file: chapters/CTF2021_background_final_version2.tex
\section{Background : optimal weak transport and Wasserstein barycenter}
\label{sec:background}

The optimal transport (OT) problem \cite{villani2008optimal} aims to find the lowest cost to transfer the mass from one probability measure onto another. Therefore, OT is a natural way to compare two probability distributions in terms of their geometric information. In particular, the Wasserstein-$p$ distance $W_p$, associated with the Euclidean cost in $\R^d$, metrises the space $\PP_p(\R^d)$ of probability measures on $\R^d$ with finite $p$-moment. Precisely, for $\mu, \nu\in\PP_p(\R^d)$,
\begin{equation}
W_p(\mu,\nu) = \left(\underset{\pi\in\Pi(\mu,\nu)}{\min} \int_{\R^d\times\R^d}\Vert x-y\Vert^p \d\pi(x,y)\right)^{1/p},
\label{def:OT}
\end{equation}
where $\pi$ is a \emph{transport plan} between   $\mu$ and $\nu$, that is,  an element of the set  $\Pi(\mu,\nu)$ of probability measures on the product space $\R^d\times\R^d$  with marginals   $\mu$ and $\nu$.  For $p=2$ and $\mu$ absolutely continuous (\emph{a.c.}), the unique optimal plan is concentrated on the graph of a measurable map called {\emph Monge map} such that $\nu=T\#\mu$, see eq.~\eqref{def:Monge} in Appendix \ref{sec:Wass_dist}.

{\bf Optimal weak transport.} We consider here the optimal weak transport (OWT) problem introduced in \cite{gozlan2017kantorovich} and in particular the special case of barycentric transport costs. The  OWT problem is then defined for $\mu, \nu \in\PP_2(\R^d)$ by
\begin{equation}
\label{def:weak_transport}
V(\mu | \nu) = \inf_{\pi\in\Pi(\mu,\nu)} \int_{\R^d} \Vert x-\int_{\R^d} y \d\pi_x(y)\Vert^2 \d\mu(x),
\end{equation}
where $\pi_x$ is the \emph{disintegration} of the transport plan $\pi$ with respect to the first marginal $\mu$, \textit{i.e.} $\pi(\d x\d y) = \pi_x(\d y)\mu(\d x)$.
As our work strongly leans on OWT theory, we recall in Appendix \ref{sec:continueV}, Th.~\ref{th:continuityV}, that $V$ is continuous with respect to the Wasserstein metric \cite{backhoff2020weak}. The optimisation problem in Eq.~\eqref{def:weak_transport} can also be reformulated thanks to the Brenier-Strassen theorem \cite{gozlan2020mixture}, \cite{backhoff2019existence},  through the notion of convex ordering. We denote by $\eta\leq_c\nu$ the {\it convex order of measures}, meaning that  $\int \phi\d\eta\leq \int\phi\d\nu$ for any convex function $\phi$ that is nonnegative or integrable with respect to $\eta+\nu$. By Strassen's theorem \cite{strassen1965existence},  two distributions are in convex order if  and only if there exists a martingale coupling between them. \CB{Additionally, the following result from \cite{backhoff2019existence}, as a generalisation of the result originally proved in \cite{gozlan2020mixture}, Th.~1.2., lays the ground for our proposed weak barycenters.}

\begin{theorem}[\cite{backhoff2019existence}, Theorem 1.4]
\label{th:1.4Julio}
Let $\mu\in\PP_2(\R^d)$ and $\nu\in\PP_1(\R^d)$. There exists a unique $\eta^{\ast}\leq_c\nu$ such that
\begin{equation}
W_2^2(\mu,\eta^{\ast}) = \inf_{\eta\leq_c\nu}\ W_2^2(\mu,\eta) = V(\mu | \nu).
\end{equation}
Moreover, there exists a convex function $\psi:\R^d\rightarrow\R$ of class $C^1$ with $\nabla \psi$ being $1$-Lipschitz, such that $\nabla\psi\#\mu = \eta^{\ast}$. \CB{Finally, an optimal coupling $\pi\in\Pi(\mu,\nu)$ for $V(\mu|\nu)$ exists, and a coupling $\pi\in\Pi(\mu,\nu)$ is optimal for $V(\mu|\nu)$ if and only if $\int y \d\pi_x(y) = \nabla\psi(x), \mu-a.s.$}
\end{theorem}
 \CB{Given a plan $\pi \in \Pi(\mu,\nu)$ achieving the minimum in Eq.~\eqref{def:weak_transport}, the measurable map, or barycentric projection,
\begin{equation}
\label{def:proj_bary}
 S_{\mu}^{\nu}(x) :=\int_{\R^d} y \d\pi_x(y)
\end{equation} 
associated to $\pi$ will be consequently called \emph{optimal barycentric projection}. Notice that, by  Theorem \ref{th:1.4Julio}, two optimal barycentric projections coincide $ \mu-a.s.$ and, in particular, they do no depend on the choice of optimal plans defining them.} With this notation, we can write the OWT cost in terms  of an OT cost according to $V(\mu | \nu)=W_2^2(\mu, S_{\mu}^{\nu}\#\mu)$.
We emphasise that $S_{\mu}^{\nu}$ is directly related to the optimisation problem \eqref{def:weak_transport}, whereas applied works such as  \cite{solomon2015convolutional, seguy2017large} make use of a barycentric projection constructed from a transport plan solving an OT problem between $\mu$ and $\nu$ (often regularised) as a substitute for the Monge map, which may not exist (more details on $S_{\mu}^{\nu}$ are displayed in Appendix \ref{sec:bar_projection}).

Last, let us note that OWT is somehow also related to the martingale OT problem developed in the stochastic finance community \cite{beiglbock2013model, alfonsi2017sampling, guyon2013nonlinear}, which puts the focus on the optimal transfer of mass between distributions assumed to be in convex order themselves.

{\bf Wasserstein barycenter.} The classical Wasserstein barycenter problem for a set of probability measures $\nu_1,\ldots,\nu_n\in\PP_2(\R^d)$ with weights $\lambda_1, \ldots, \lambda_n$ in the simplex (i.e. $\lambda_i\geq 0$ and $\sum_{i=1}^n\lambda_i = 1$) is defined \cite{agueh2011barycenters} by
\begin{equation}
\label{eq:barycenter}
\argmin_{\mu\in\PP_2(\R^d)} \sum_{i=1}^n \lambda_i W_2^2(\mu,\nu_i).
\end{equation}

The Wasserstein barycenter has been extensively studied both theoretically and numerically \cite{le2017existence, zemel2019frechet, alvarez2016fixed, bigot2018characterization}. Regarding the numerical part, \cite{staib2017parallel} focuses on the computation of Wasserstein barycenters for a fixed number of measures and a stream of observations per measure; additionally, \cite{li2020continuous} proposed an entropy-regularised alternative via stochastic optimisation for computing the Wasserstein barycenter of \emph{a.c.} distributions only from observations.
 Constrained by their assumption of \emph{a.c.}, \cite{zemel2019frechet} computes the Wasserstein barycenter by smoothing the observed empirical distributions. Furthermore, \cite{dvinskikh2020stochastic} compares the complexity of both the \emph{sample} Wasserstein barycenter and a stochastic approximation to estimate a population barycenter (discrete measures and entropic regularisation). Finally, the authors of \cite{altschuler2021wasserstein} recently proposed an algorithm to compute the barycenters in polynomial time.

%% file: chapters/CTF2021_section_weakBarycenter_final_version2.tex
\section{Optimal weak transport barycenters and latent variable interpretation} 
\label{sec:weak_bar}  

\subsection{Definition and basic properties}

In a similar fashion, based on the weak transport cost in Eq.~\eqref{def:weak_transport}, we propose the following variant:
\begin{definition}
\label{def:weak_barycenter}
The set of weak barycenters of a finite family of measures $\{\nu_i\}_{i=1,\ldots,n}\in\PP_2(\R^d)$ with weights $\{\lambda_i\}_{i=1,\ldots,n}$ in the simplex is defined as
\begin{equation}
\label{eq:weak_barycenter}
\argmin_{\mu\in\PP_2(\R^d)} \sum_{i=1}^n \lambda_i V(\mu | \nu_i).
\end{equation}
\end{definition}
Thus, a weak barycenter  averages, with respect to the Wasserstein metric, an optimally chosen set of probability measures  $\{\eta_1,\ldots,\eta_n\}$  which are \emph{more concentrated} than the corresponding $\nu_i$, in the sense that $\eta_i\leq_c\nu_i$ for each $1\leq i \leq n$. The existence of a solution  is established as follows:
\begin{proposition}
\label{prop:existence_weak_bar}
The weak barycenter problem in Eq.~\eqref{eq:weak_barycenter} admits a minimiser $\mu\in\PP_2(\R^d)$.
\end{proposition}
See Sec.~\ref{sec:proofs_weakbar} of the Appendix for the proof of the above Proposition (which relies on Prokhorov's theorem) and all the proofs for this Section. Uniqueness is in general not granted: we next show that the set of solutions is indeed an interval, with respect to the partial order of convex ordering of probability measures. 
  
In the following, we denote by $X$ and $Y_i$ random variables with respective laws $\mu$ and $\nu_i$, for $1\leq i\leq n$, and $\delta_a$ the Dirac measure supported on $a\in\R^d$.
\begin{lemma}
\label{lemma:diracweak}
If $\mu$ is a weak barycenter of $\{\nu_i\}_{i=1\ldots,n}$ and $\mu' \leq_c \mu$, then $\mu'$ also is a weak barycenter. 
In particular, the Dirac measure supported on $\E_\mu (X)$ is always a weak barycenter. Moreover, a Dirac distribution $\delta_{\bar{\omega}}$ is a weak barycenter if and only if $\bar{\omega}=\sum_{i=1}^n\lambda_i\E_{\nu_i} (Y_i)$. 
\end{lemma}

A  consequence of the above lemma is that for any weak barycenter $\mu$, 
\begin{equation}
\label{eq:meanwB}
\E_\mu(X)=\sum_{i=1}^n \lambda_i\E_{\nu_i}(Y_i) \, ,
\end{equation}
and the value of the weak barycenter problem is given by
\begin{equation}
\label{eq:obj_function}
\inf_{\mu\in\PP_2(\R^d)} \sum_{i=1}^n \lambda_i V(\mu | \nu_i) =  \sum_{i=1}^n \lambda_i \Vert \E(Y_i)\Vert^2-\Vert  \sum_{i=1}^n \lambda_i \E(Y_i)\Vert^2.
\end{equation}

We can also derive the following characterisation on the set of weak barycenters: 
\begin{proposition}
\label{prop:bary_charact}
A measure $\mu\in {\cal P}(\R^d)$ is a weak barycenter of  $\{\nu_i\}_{i=1\ldots,n}$  if and only if its mean satisfies \eqref{eq:meanwB} and $\hat{\mu}\leq_c \hat{\nu}_i$ holds for all $1\leq i\leq n$, where $\hat{\nu}$ denotes the centered version of a  law $\nu$. 
\end{proposition}
For instance,  in the case of one dimensional Gaussian distributions   $\nu_i = \NN(m,\sigma^2_i)$,  the set of weak barycenters includes $\{\mu=\NN(m,\sigma^2)\ | \ 0\leq\sigma^2\leq\min_{1\leq i\leq n} \sigma^2_i\}$.

A natural question is whether a "maximal'' weak barycenter exists, in the sense of convex ordering (up to translation by the mean). For $d=1$, the answer is affirmative. When the means $\E(Y_i)$ are equal, this  follows from the complete lattice property of the set of probability measures  with respect to the convex ordering (see \cite{kertz2000complete}); the general case can then be reduced to the latter using Proposition  \ref{prop:bary_charact}. For $d\geq 2$,  this property is in general not true and the answer depends on the family $\{\nu_i\}_{i=1\ldots,n}$.

In the particular case of \emph{a.c.} input measures, we can bound the distance between the Wasserstein and weak barycenters by the variances of the distributions $(\nu_i)_{1\leq i \leq n}$. The barycenters are then closer the more concentrated each $\nu_i$ is.
\begin{lemma}
\label{lemma:weakVSwass}
Let $\nu_1,\ldots,\nu_n\in\PP_2(\R^d)$ be {\it a.c.}, at least one of them with bounded density. Let $\bar{\mu}$ and  $\tilde{\mu}$  respectively denote the weak and the Wasserstein barycenters. Then
\[W_2^2(\bar{\mu},\tilde{\mu})\leq 2 \sum_{i=1}^n \lambda_i \left(\E \Vert Y_i\Vert^2-\Vert\E Y_i\Vert^2\right).\]
\end{lemma}

\subsection{Weak barycenters as  latent variables}
\label{sec:latent_variable}

The weak barycenter encodes common geometric information present in all the input measures considered, therefore, it can be intuitively and rigorously interpreted as being the distribution of a latent variable underlying the realisations of random variables of laws $\nu_i$ for all $ 1\leq i \leq n$.
\begin{theorem}
\label{theo:latent}
Let $\mu$ be a weak barycenter of $\{\nu_i\}_{i=1\ldots,n}$. Then,  for each $1\leq i\leq n$, a random variable $Y_i\sim \nu_i$ can be realised as
$$ Y_i=  X  +  ( \E Y_i  -  \E X ) + \bar{Y}_i, $$ 
where $X\sim \mu$ and $\bar{Y}_i= Y_i-  \E( Y_i \vert X) $ is centered  conditionally on $X$.  Moreover,  one has $S_{\mu}^{\nu}(X) = X  +  ( \E Y_i  -  \E X ) $ for all $i=1,\dots, n$. Finally, we have $\E( Y_i- \E Y_i \vert X -\E X ) = X -\E X $ or, equivalently, $\hat{\mu}\leq_c \hat{\nu}_i$,  with $\hat{\mu}$  and $\hat{\nu}_i$ the laws of $X -\E X$ and $Y_i- \E Y_i $ respectively. 
\end{theorem}

That is to say, each $Y_i\sim \nu_i$ can be realised by sampling a random variable  $X$ common  to all $i=1,\dots, n$ and distributed according to the weak barycenter $\mu$,  translating that value  by   $\E Y_i  -  \E X $ and adding a \emph{cluster-specific} component $\bar{Y}_i $ or idiosyncratic  noise,  centered conditionally on $X$. 


\begin{remark}
\label{remark:robustness}
The observations of each class (i.e. input measure) can be interpreted as outliers with respect to the (translated) law of the weak barycenter, which are statistically different and are thus left aside of its support. This way, the weak barycenter is robust to outliers, as it tends to discard them, by construction. Furthermore, this "robustness" property results in the stability of weak barycenter upon perturbation of a class with larger noise (or more scattered, outlying values). More precisely, if a class is corrupted in such a way that their observations result in a stochastically larger distribution than the original one, a weak barycenter computed in terms of the original (stochastically smaller) class will still be a weak barycenter in the new corrupted setting. An intuitive and simple way to illustrate this point follows by considering a weak barycenter $\mu$ of a one dimensional and centered family of input distributions $\{\nu_i\}_{i=1,\ldots,n}$. By Proposition \ref{prop:bary_charact}, $\mu$ must verify $\mu\leq_c \nu_i$ for all $i=1,\ldots,n$. In particular, from Theorem 3.A.1.~in \cite{shaked2007stochastic}, we have that $\int_x^{\infty}\P(X>u)\d u \leq \int_x^{\infty}\P(Y_i>u)\d u$ for all $x\in\R$, where $X\sim\mu$ and $Y_i\sim\nu_i$. Therefore, $\mu$ is likely to avoid outliers. Another supportive intuition in terms of robustness is that a maximal weak barycenter would be one that includes the most possible points of all classes (or distributions) in its support (all this, after re-centering) and leaves out only "outliers". A non-maximal weak barycenter is then more conservative, meaning that it counts on fewer points and leaves out more possible outliers.
\end{remark}

%% file: chapters/CTF2021_section_weakPopBarycenter_final_version2.tex
\subsection{Extension for the population barycenter}
\label{sec:pop_barycenter} 

The population Wasserstein barycenter introduced in \cite{le2017existence} and \cite{alvarez2015wide} extends the definition of Wasserstein barycenter for an infinite number of measures. This formulation is particularly relevant for the construction of an iterative algorithm to compute the barycenter for the streaming case, that is, when the measures are received \emph{online}. The proofs are reported in Section \ref{sec:proofs_popweakbar} of the Appendix.

Let us consider a probability measure $\Q\in\PP_2(\PP_2(\R^d))$, meaning that $\Q$ is supported on a set of measures with finite moments of order $2$, such that for some (and thus all) $\mu\in\PP_2(\R^d)$, we have that $ \int_{\PP_2(\R^d)} W_2^2(\mu,\nu)\d\Q(\nu) <\infty$. 
\begin{definition}
\label{def:popweakbar}
We define the set of weak population barycenters of a distribution $\Q\in\PP_2(\PP_2(\R^d))$ as
\begin{equation}
\label{eq:weak_pop_barycenter}
\argmin_{\mu\in\PP_2(\R^d)} \int_{\PP_2(\R^d)}V(\mu|\nu)\d\Q(\nu).
\end{equation}
\end{definition}
The following lemma guarantees that the map $(x,\nu)\mapsto S_{\mu}^{\nu}(x)$ appearing in Eq.~\eqref{eq:weak_pop_barycenter} through $V(\mu|\nu)=\int\Vert x- S_{\mu}^{\nu}(x)\Vert^2\d\mu(x)$ is well defined.

\begin{lemma}
\label{lemma:Smeasurable}
\CB{For each $\mu \in \PP_2(\R^d)$, the function $(x,\nu)\in  \R^d \times \PP_2(\R^d)\mapsto  S_{\mu}^{\nu}(x)$, constructed with an optimal plan $ \pi$ realising  $V(\mu | \nu)$ in Eq.~\eqref{def:weak_transport}, is measurable.}
\end{lemma}

Using similar arguments as those of Proposition \ref{prop:existence_weak_bar} and the fact that any probability measure can be approximated by a sequence of probability measures with finite support, the following proposition confirms that the weak population barycenter problem is also well defined.

\begin{proposition}
\label{prop:existence_pop_bar}
The minimisation problem in Eq.~\eqref{eq:weak_pop_barycenter} admits a solution.
\end{proposition}

%% file: chapters/CTF2021_algorithms_final_version2.tex

\section{Algorithms via fixed-point representations}
\label{sec:algorithms}

\subsection{Weak barycenter}

For the Wasserstein barycenter problem in Eq.~\eqref{eq:barycenter}, the authors in \cite{agueh2011barycenters}  proved that if at least one of the measures $\nu_1,\ldots,\nu_n$ is \emph{a.c.}, the Wasserstein barycenter is unique. Furthermore, if all the $\nu_i$'s are \emph{a.c.}, and at least one of them has a bounded density, then the unique Wasserstein barycenter is also \emph{a.c.} and verifies a fixed-point equation. This last property has been thoroughly studied by \cite{alvarez2016fixed} and \cite{zemel2019frechet} and leveraged to compute an approximation of the barycenter via an iterative algorithm based on Monge maps, whose existence and uniqueness are guaranteed by the \emph{a.c.} of the measures involved.

Akin to the fixed-point methodology in the classical Wasserstein scenario, we define an iterative procedure based on the barycentric projection computed in the optimal weak transport problem in Eq.~\eqref{def:weak_transport}, that is valid for arbitrary distributions. Therefore, we consider the following iterative rule for probability measures $\nu_1,\ldots,\nu_n\in\PP_2(\R^d)$:
\begin{equation}
\label{def:G}
\mu_{k+1} = G(\mu_k),\ \mbox{with}\ G(\mu)= \left(\sum_{i=1}^n \lambda_i S_{\mu}^{\nu_i}\right)\# \mu,
\end{equation}
where for each $i=1,\ldots,n$, the optimal barycentric projection is given by $S_{\mu}^{\nu_i}:x\mapsto\int y\d\pi_x^{\mu,\nu_i}(y)$, for any $\pi^{\mu,\nu_i}\in\Pi(\mu,\nu_i)$ achieving the minimum in the OWT problem in Eq.~\eqref{def:weak_transport}. The  proposed iterative procedure is presented in Algorithm \ref{algo:weakbar}.

A fundamental difference between the fixed-point computation of the Wasserstein barycenter \cite{alvarez2016fixed} and a weak barycenter is that the optimal Monge map $T_{\mu}^{\nu}$ in the OT problem verifies $T_{\mu}^{\nu}\#\mu=\nu$, whereas the pushforward measure $S_{\mu}^{\nu}\#\mu$ in the OWT setting still depends on $\mu$. We will then prove that the iterative algorithm in Eq.~\eqref{def:G}, based on the maps $S_{\mu}^{\nu_i}$, admits converging subsequences. A convenient result is the continuity of the functional $G$ in Eq.~\eqref{def:G}, which can be proven using Arzela-Ascoli theorem on a set of barycentric projections as well as the Skorohod's representation theorem.
\begin{theorem}
\label{th:continuityG}
The function $\mu\mapsto G(\mu)$ defined in Eq.~\eqref{def:G} is $W_2$-continuous from $\PP_2(\R^d)$ to $\PP_2(\R^d)$.
\end{theorem}

Using an approach similar to \cite{alvarez2016fixed} for the Wasserstein barycenter, we can state the following results for the proposed fixed-point procedure.
\begin{proposition}
\label{prop:fixedpoint}
If $\mu$ is a weak-barycenter, that is a solution of problem \eqref{eq:weak_barycenter}, then $G(\mu) = \mu$ i.e.  $x = \sum_{i=1}^n \lambda_i S_{\mu}^{\nu_i}(x), \mu(x)$-a.s.
\end{proposition}

The inverse implication of Proposition \ref{prop:fixedpoint}  is not necessarily true, that is, some fixed points may not be weak barycenters. However, a Dirac delta $\delta_{\omega}, \omega\in\R^d,$ that meets the fixed-point condition  $\delta_{\omega} = G(\delta_{\omega})$, is a weak barycenter (see Lemma \ref{lemma:diracweak}).

\begin{proposition}
\label{prop:convergence_finite}
Let $(\mu_k)_k$ be the sequence defined by the iterative procedure $\mu_{k+1} = G(\mu_k)$ and starting from $\mu_0\in\PP_2(\R^d)$. Then $(\mu_k)_k$ is tight and every converging subsequence must converge to a fixed point of $G$. 
\end{proposition}

We observe that these results also hold for the classical Wasserstein barycenter of \emph{a.c.} measures $\{\nu_i\}_{i=1\ldots,n}$ such that at least one of them has a bounded density. Moreover, the inverse implication, namely if $\mu$ is a fixed-point then it is a barycenter, is not straightforward  even in the Wasserstein barycenter case, for which one considers the fixed-point equation given by $\mu = (\sum_{i=1}^n \lambda_i T^{\nu_i}_{\mu})\#\mu$, with $T^{\nu_i}_{\mu}$ the Monge map verifying $\nu_i = T^{\nu_i}_{\mu}\#\mu$. Indeed, \cite{agueh2011barycenters} prove that if $\mu$ checks $x = \sum_{i=1}^n \lambda_i T^{\nu_i}_{\mu}(x)$ for every $x\in\R^d$, not only $\mu$-almost everywhere, then $\mu$ is a Wasserstein barycenter. Also, \cite[Theorem 2]{zemel2019frechet} provide additional conditions for this to be true by essentially invoking more smoothness on the distributions $\{\nu_i\}_{i=1\ldots,n}$. Additionally, they only conjecture that under the same assumptions, the fixed-point is unique. Our method, however, includes arbitrary probability measures. Therefore, we do not expect to obtain similar results as in the Wasserstein barycenter case, for which smoothness is required.

\subsection{Weak population barycenter}

Based on \cite{bakchoff2022}, we construct a stochastic iterative algorithm for computing the weak population barycenter in Eq.~\eqref{eq:weak_pop_barycenter}. We clarify that \cite{bakchoff2022} is constrained to probability measures $\Q$ supported on distributions that are \emph{a.c.}, whereas in our setting these distributions only need to belong to $\PP_2(\R^d)$. Let us notice that our algorithms can be interpreted as geodesic gradient descent as in \cite{bakchoff2022} and \cite{chewi2020gradient}, however, OWT is not a metric and its potential geodesic structure is so far unknown. Therefore, the proposed algorithm only aims to mimic Riemmanian gradient descent.
Our fixed-point result for the weak population barycenter problem is stated in the following Lemma:

\begin{lemma}
\label{lemma:fixed_point}
If $\mu$ is a weak population barycenter of $\Q$, then $x=\int S_{\mu}^{\nu}(x)\d\Q(\nu), \mu(x)$-a.s. 
\end{lemma}

As in the finite case, the inverse implication is difficult to obtain. In particular, this has not been proven for the classical population Wasserstein barycenter in \cite{bakchoff2022}, where it boils down to prove the uniqueness of an absolutely continuous fixed point of $\mu\mapsto(\int T_{\mu}^{\nu}\d\Q(\nu))\#\mu$, where $T_{\mu}^{\nu}$ is the Monge map between $\mu$ and $\nu$. As explained in \cite{bakchoff2022}, the uniqueness of such fixed points has also been studied under some strong assumptions in \cite{bigot2018characterization} by considering parametric classes of random probability measures with compact support. This result is expected to be true by again	 invoking more smoothness on the distributions at hand. As our method focuses (in particular) on discrete probability measures, the conditions under which the inverse implication  holds are beyond the scope of our work. However, from the experimental results in Section \ref{sec:expe}, we believe our method presents practical advantages.

We next present an iterative scheme converging  towards a distribution $\mu$ verifying the fixed-point equation in Lemma \ref{lemma:fixed_point}.  This scheme is illustrated below in Algorithm \ref{algo:weakbarpop}. To prove its convergence, we will need a technical assumption on $\Q$:

\begin{minipage}{1cm}
\vspace{-0.4cm}(A)
\end{minipage}
\begin{minipage}{12cm} There exists $\epsilon>0$ and $R>0$ such that $\Q$ gives full measure to the set \qquad $\displaystyle{K_\Q:=\{\mu \in \PP_2(\R^d): \int |x|^{2+\epsilon}\d \mu(x) \leq R\}}$.\end{minipage}




\begin{definition}
Let $\mu_0\in K_\Q, \nu^k \overset{i.i.d.}{\sim}\Q$ and $\gamma_k>0$. We define the following iterative procedure:
\begin{equation}
\label{eq:iterative_algo}
\mu_{k+1} = \left[(1-\gamma_k)\textnormal{id}+\gamma_k S_{\mu_k}^{\nu^k}\right]\#\mu_k, \  k\geq 0,
\end{equation}
where $S_{\mu_k}^{\nu^k}$ is the optimal barycentric projection between $\mu_k$ and $\nu^k$ and \textnormal{id} is the identity operator.
\end{definition}
 The following standard conditions on the steps $\gamma_k$ will also be assumed:
\begin{equation}
\label{cond:gamma}
\sum_{k=1}^{\infty} \gamma_k^2 <\infty\qquad \mbox{and} \qquad \sum_{k=1}^{\infty} \gamma_k = \infty,
\end{equation}

\begin{theorem}
\label{th:conv_iter}
Assume Conditions in Eq.~\eqref{cond:gamma}, (A) and moreover that every measure verifying the fixed-point equation $x=\int S_{\mu}^{\nu}(x)\d\Q(\nu), \mu(x)$-a.s. is a weak barycenter. Then the sequence $(\mu_k)_k$ in Eq.~\eqref{eq:iterative_algo} is a.s. relatively compact w.r.t. $W_2$ and every limit point is a weak barycenter.
\end{theorem}
The proof, provided in the supplementary material, is inspired by the standard Wasserstein barycenter case studied in \cite{bakchoff2022}, \cite{rios1805bayesian}. Assumption (A) grants that the sequence in Eq.  \eqref{eq:iterative_algo} remains in some compact set, and can be replaced by more general conditions (see Remark \ref{remark:assumptionA} in Appendix).

\begin{minipage}[t]{0.48\textwidth} 
        \begin{algorithm}[H]
{\bf Input:} distributions $\nu_1,\ldots,\nu_n$, \# steps $K$;\\
initialisation: $\mu_0=\nu_1$;\\
 \For{$k=0,1,\ldots,K$}{
\For {$i=1,2,\ldots,n$}{
Solve the OWT problem between $\mu_k$ and $\nu_i$ to obtain $\pi^{\mu_{k},\nu_i}$;\\
  $S_i = \int y\d\pi_x^{\mu_{k},\nu_i}(y)$
 }
$\mu_{k+1}=(\sum_{i=1}^n \lambda_i S_i)\#\mu_{k}$
 }
 \caption{Weak barycenter}
 \label{algo:weakbar}
\end{algorithm}
\end{minipage}
\hspace{2em}\begin{minipage}[t]{0.48\textwidth} 
        \begin{algorithm}[H]

{\bf Input:}  number of steps $K$;\\
initialise distribution $\mu_0\sim\Q$;\\
 \For{$k=0,1,\ldots,K$}{
     Sample $\nu^k \sim \Q$;\\
     Update $\gamma_k$;\\
     Solve the OWT problem to obtain $\pi^{\mu_{k},\nu^k}$\\
	$S_k = \int y\d\pi_x^{\mu_{k},\nu^k}(y)$;\\
 	$\mu_{k+1} =\left[(1-\gamma_k)\id +\gamma_k S_k\right]\#\mu_{k}$;\\
 }
 \caption{Weak population barycenter}
  \label{algo:weakbarpop}
\end{algorithm}

\end{minipage}

%% file: chapters/CTF2021_experiments_final_version2.tex
\section{Computational aspects}
\label{sec:computation}

{\bf Setting and computation of OWTs.} Both Algorithms \ref{algo:weakbar} and \ref{algo:weakbarpop} require the computation of the optimal barycentric projection associated to the OWT problem in Eq.~\eqref{def:weak_transport}. For two discrete measures $\mu=\sum_{i=1}^r a_i\delta_{x_i}$ and $\nu=\sum_{j=1}^m b_j\delta_{y_j}$, this boils down to solving the following quadratic programming problem
\begin{equation}
\label{weakOT_discrete}
\min_{\pi\in\R^{r\times m}} \left\{\sum_{i=1}^r a_i  \left\Vert x_i- \left(\frac{\pi {\bf y}}{\bf a}\right)_i\right\Vert^2, \pi_{ij}\geq 0, \ \pi\1 = a, \ \pi^T\1 = b\right\},	
\end{equation}
which can be solved using a solver such as {\bf cvxpy}. We also propose to solve the OWT problem in Eq.~\eqref{weakOT_discrete} with a proximal algorithm. The optimal barycentric projection is then constructed as $\frac{\pi{\bf y}}{{\bf a}}$. The details and examples are presented in Appendix \ref{sec:appendix_proximal}.

{\bf Comparison setting.} In the next section, we compare our proposed computation for weak barycenters in Definition \ref{def:popweakbar} (Algorithm \ref{algo:weakbarpop}) to the classic Wasserstein barycenter in particular for a stream of a.c.~measures.
Namely, we will run Algorithm \ref{algo:weakbarpop} by, following \cite{cuturi2014fast,seguy2017large}, replacing optimal barycentric projections by the barycentric projections associated either to i) an optimal plan in the Kantorovich problem \eqref{def:OT}, or ii) the optimal Sinkhorn plan in the entropy regularised OT problem \cite{sinkhorn_lightspeed} given by
\begin{equation}
\label{def:sinkhorn}
\argmin_{\pi\in\Pi(\mu,\nu)} \int \Vert x-y\Vert^2\d\pi(x,y) + \varepsilon KL(\pi | \mu \otimes\nu),
\end{equation}
where KL denotes the Kullback-Leibler divergence. The associated barycenters will be referred to as \emph{OT barycenter} and \emph{OT Sinkhorn barycenter} respectively. The optimal plans for OT and regularised OT problem were computed using \emph{POT toolbox} \cite{flamary2021pot}. Notice that what we call OT barycenter (resp.~OT Sinkhorn barycenter) is not solving a Wasserstein barycenter problem (resp.~a regularised Wasserstein barycenter problem).  Therefore, our method for barycentric computation differs from previous ones in the literature (see Section \ref{sec:background}) in that it i) can process a {\it stream} of an unknown number of measures, ii) does not require the measures to be a.c., and iii) does not appeal to additional regularisation of the measures or the Wasserstein metric.

\section{Experimental results}
\label{sec:expe}

This section is devoted to the empirical validation of our proposal on both synthetic and real-world data. We first focused on Algorithm \ref{algo:weakbarpop} since multiple algorithms to compute a Wasserstein barycenter for a fixed number of distributions are already available \cite{cuturi2014fast,staib2017parallel}. We present two robustness to outliers experiments, then we validate our OWT barycenter on synthetic dataset and real-world ones. The overall conclusion of our experiments is that the weak barycenter is more likely to maintain the common (or shared) geometric features of the measures involved, as expected from Theorem \ref{theo:latent}. Additional experiments are presented in Appendix \ref{sec:additional_expe}, including the comparison of the energy for the computed weak barycenter in Algorithm \ref{algo:weakbar} against the approximated optimal energy (using Eq.\eqref{eq:obj_function} and the plug-in estimator).

\subsection{Robustness to outliers}
\label{sec:outliers}

OT's sensitivity to outliers is a well-known problem that can be addressed  {\emph e.g.} with unbalanced OT \cite{balaji2020robust}. We observed that OWT also allows to deal with outliers, which is coherent with the latent variable interpretation (see Remark \ref{remark:robustness}). We illustrate this with two experiments. In Fig.~\ref{fig:robust_outliers} (left), we consider $50$ sets of $20-30$ observations from different 2D Gaussian measures, where each observation may be corrupted by random translations (Bernoulli $p = 0.05$) thus producing outliers. We show the resulting barycenters (dots), and barycenters without outliers (crosses) for Wasserstein barycenter (red) and weak barycenter (black), which shows robustness to outliers. In Fig.~\ref{fig:robust_outliers} (right), we consider two distributions supported on pair-of-ellipses, and $120$ observations per distribution. Again, each observation may be corrupted by random translations (Bernoulli $p = 0.05$). The weak barycenter (black) shows a better preservation of the shapes than the Wasserstein barycenter (red), in particular, the red dots are more often located outside the ellipses.
\begin{figure}
\centering
\begin{tabular}{C{3.1cm} C{3.2cm} | C{3cm} C{2.9cm}}
\includegraphics[scale=0.18]{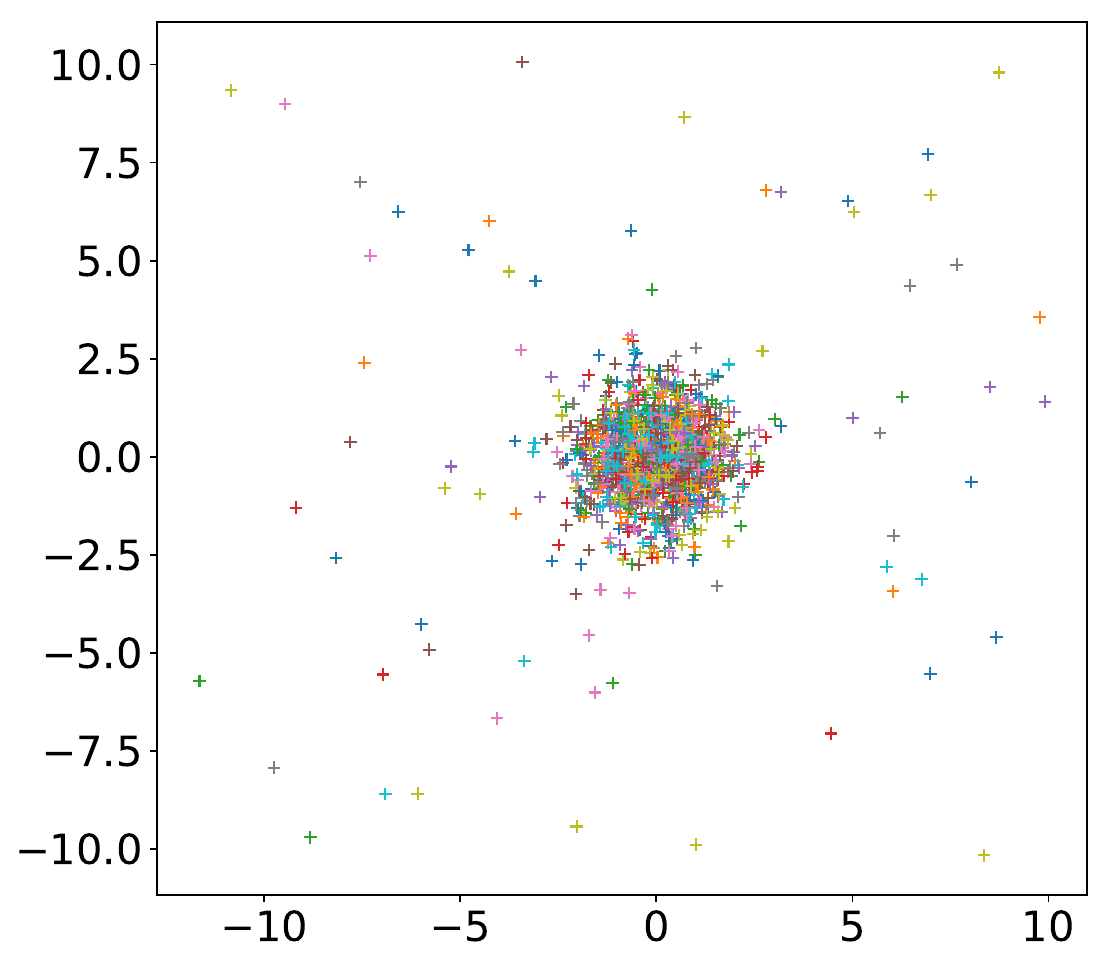} &
\includegraphics[scale=0.18]{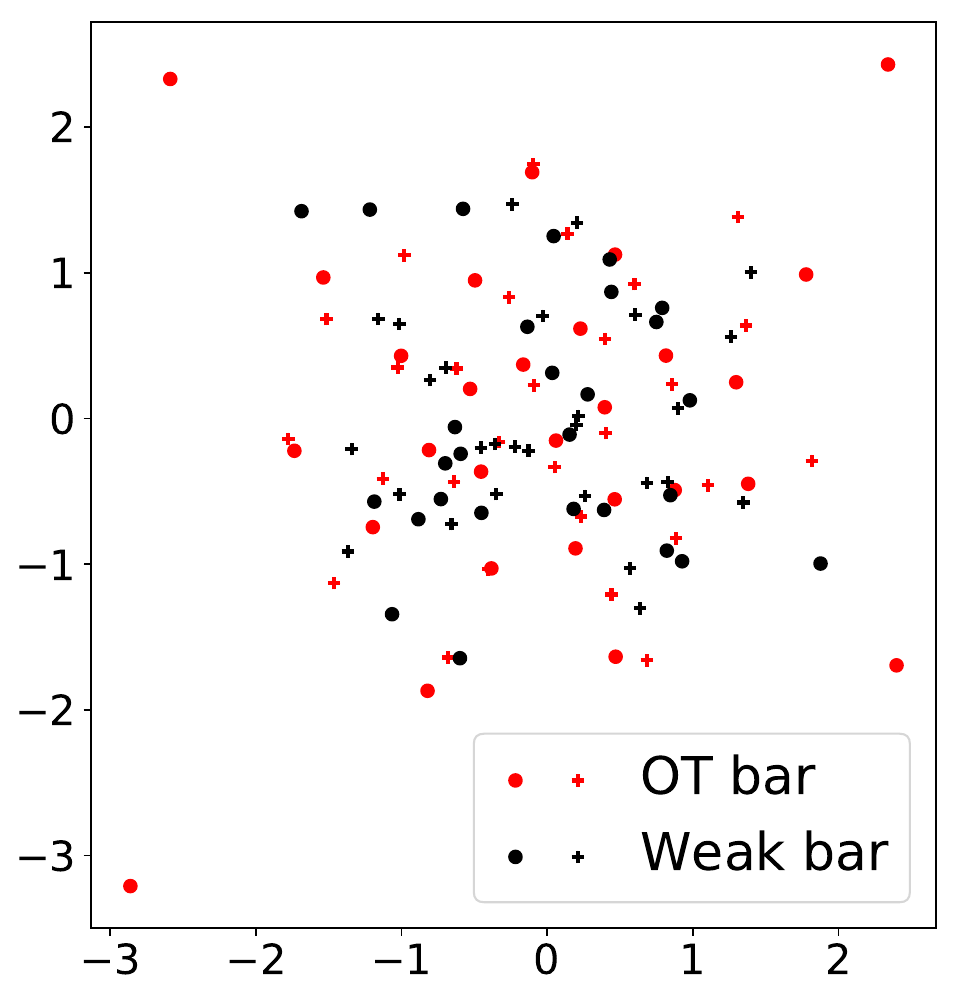} &
\includegraphics[scale=0.18]{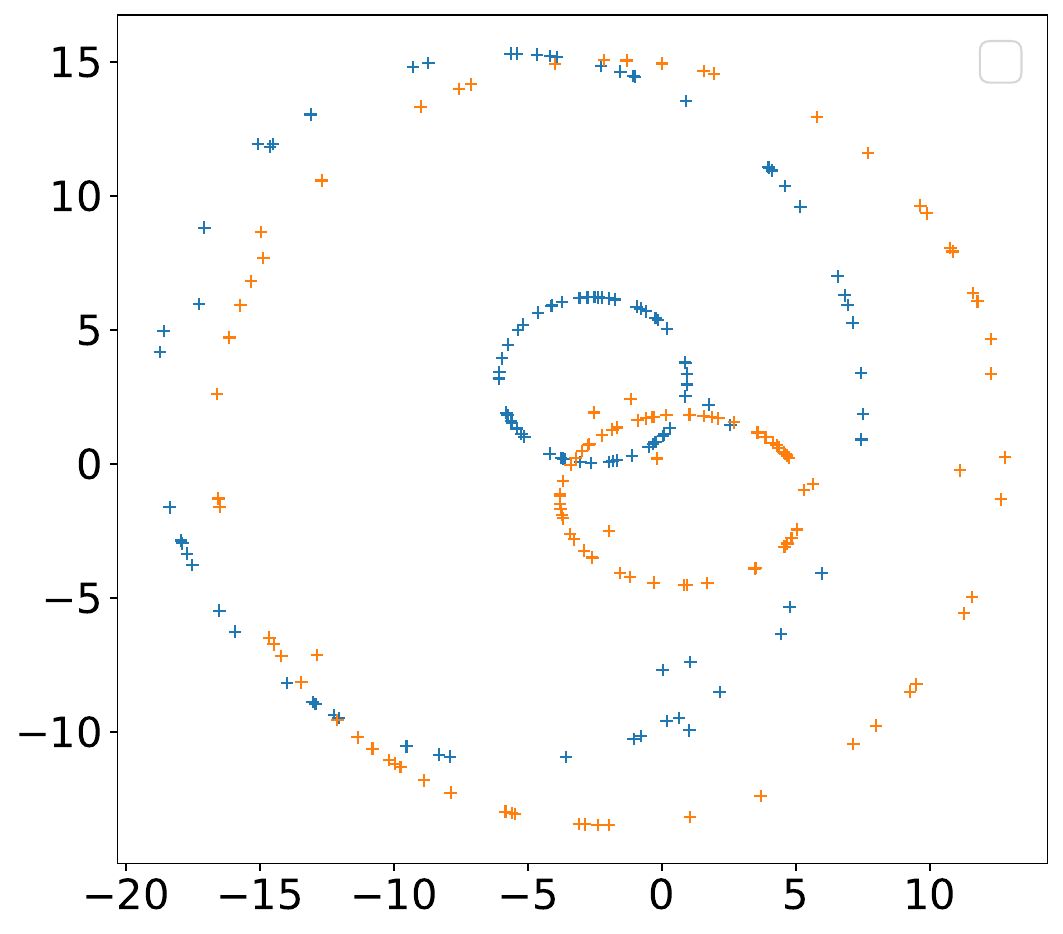} &
\includegraphics[scale=0.18]{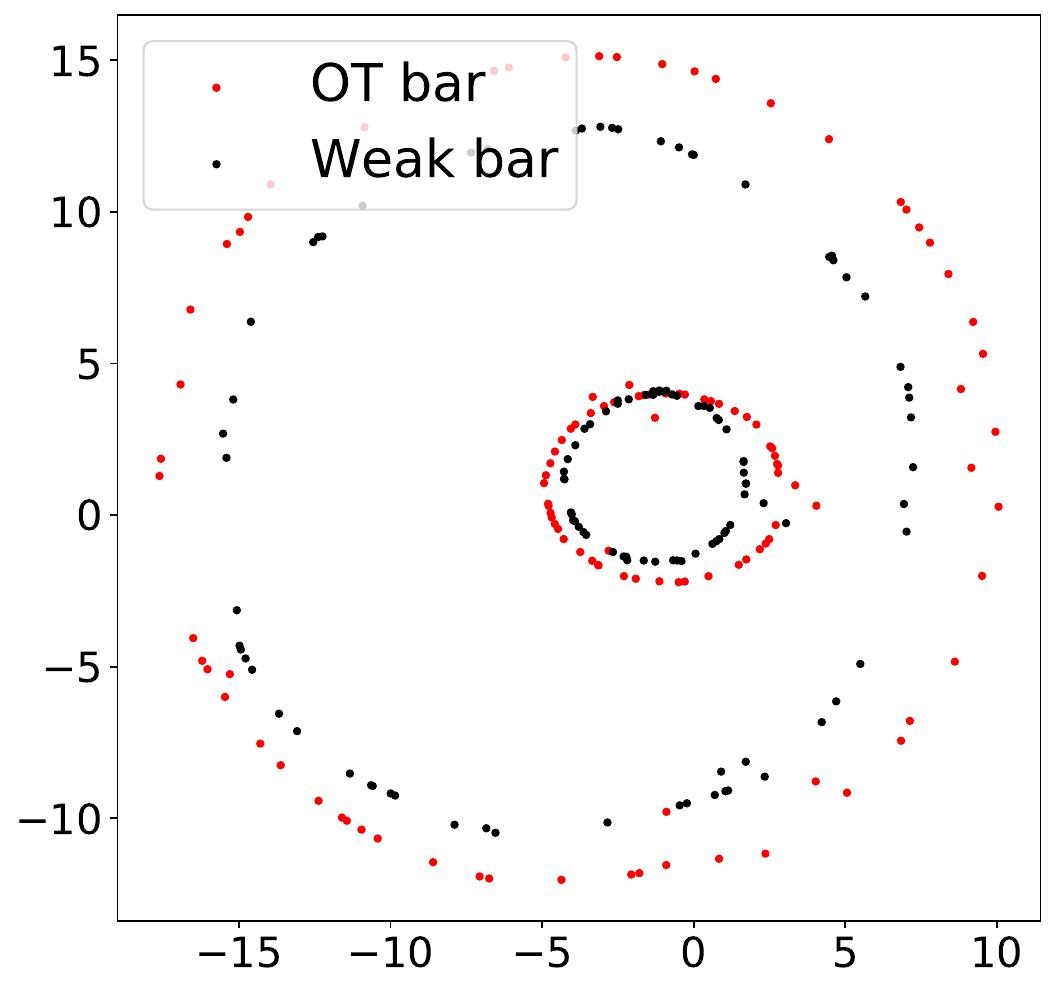} 
\end{tabular}
\caption{(left) Empirical Gaussian distributions and their OWT (black) and OT (red) barycenters for Gaussian observations (crosses) and corrupted observations (dots). (right) Empirical distributions supported on two ellipses and their OWT (black) and OT (red) barycenters.}
\label{fig:robust_outliers}
\end{figure}

\subsection{Synthetic distributions} 
\label{sec:synthetic_distribution}
We implemented the proposed sequential computation of weak barycenters (Algorithm \ref{algo:weakbarpop}) on two examples of synthetic distributions: Gaussians and spirals. In each case, we sampled $r$ observations from a random distribution at each step, and considered $K$ steps (and thus $K$ measures for each case).

{\bf 2D Gaussians} ($r=100$ \& $K=15$). We considered  distributions $\mathcal{N}(m,I)$, with $m$ uniformly distributed on $(-3,3)\times(-5,-5)$ and $I$ the identity matrix. Fig.~\ref{fig:gaussian} (left) shows the empirical distributions together with the OWT and OT barycenters, the weak barycenter being the less spread out as expected. The three remaining plots illustrate the behaviour of the barycenters constructed as stated in Sec.~\ref{sec:computation}. For a small regularisation parameter $\varepsilon$ in Eq.~\eqref{def:sinkhorn}, the OT and OT Sinkhorn barycenters are similar, however, as $\varepsilon$ increases the OT Sinkhorn (OTS) barycenter becomes closer to the weak barycenter and thus even more concentrated, meaning that its samples tend to be closer to each other. Critically, for a very large $\varepsilon$, as the entropy tends to spread the mass in the  regularised optimal plan, the associated barycentric projection will roughly move the mass to the spatial mean of the target distribution's support.
\begin{figure}[ht!]
\hspace{-0.2cm}\begin{tabular}{C{3.05cm} C{3.05cm} C{3.05cm} C{3.05cm}}
\includegraphics[scale=0.2]{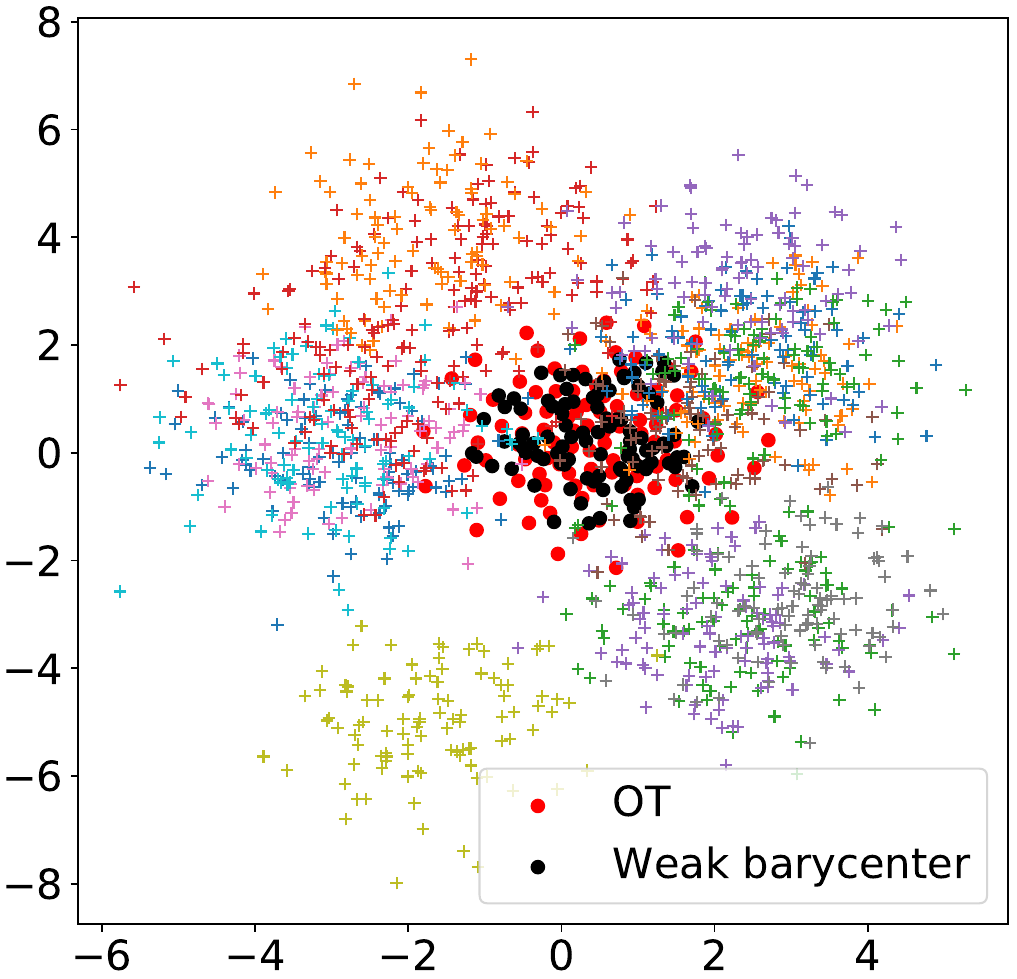} &
\includegraphics[scale=0.2]{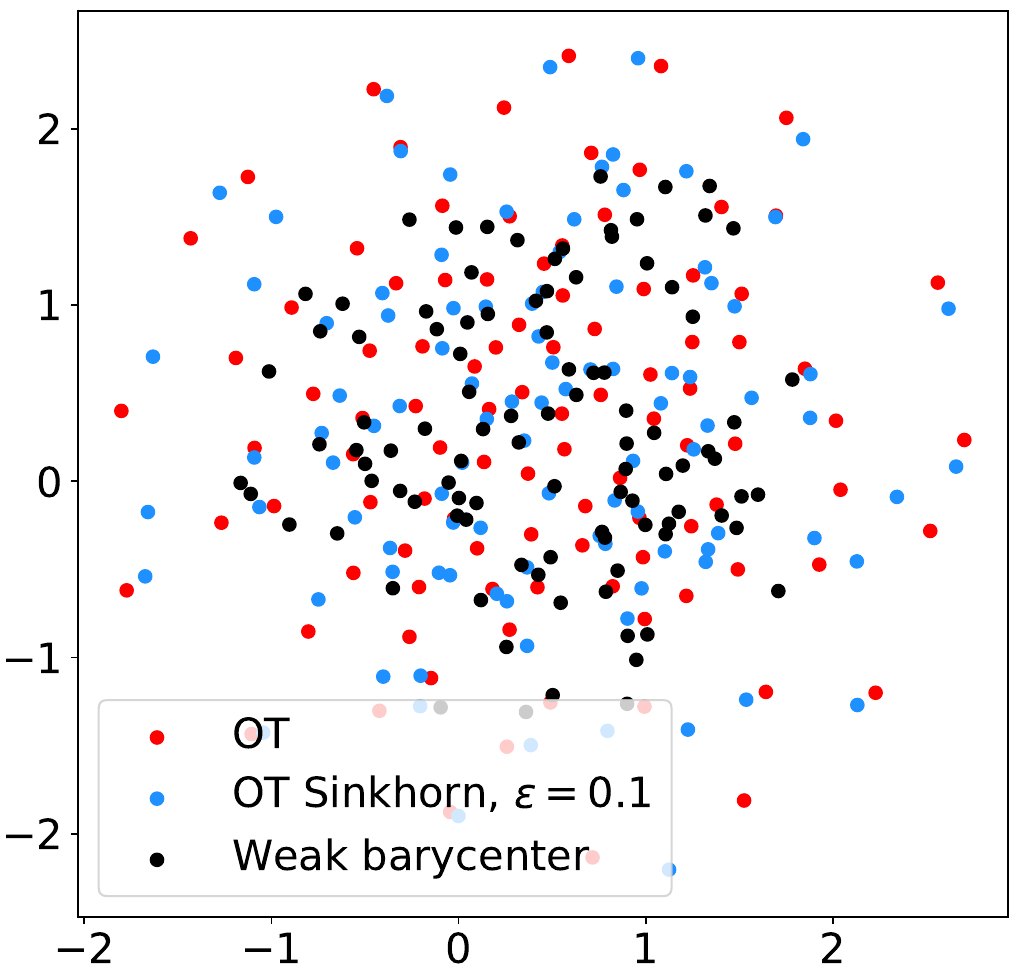} &
\includegraphics[scale=0.2]{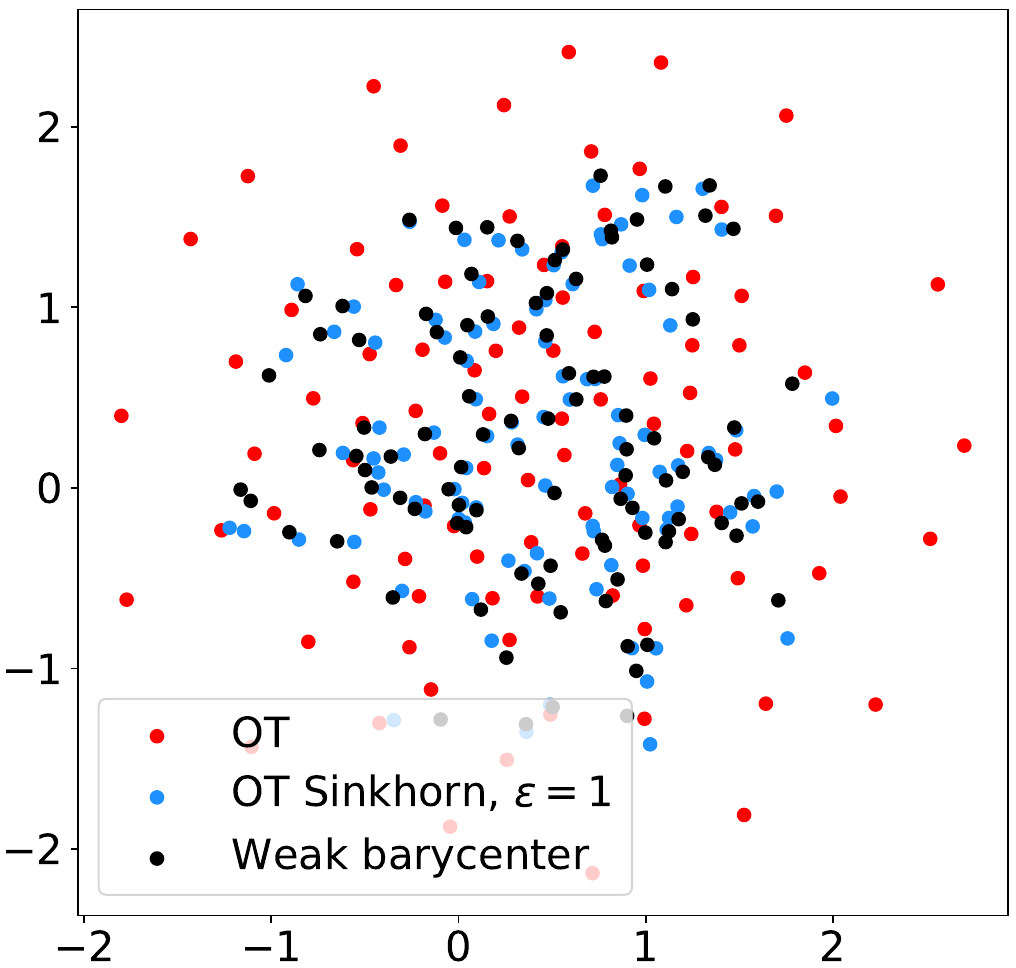} &
\includegraphics[scale=0.2]{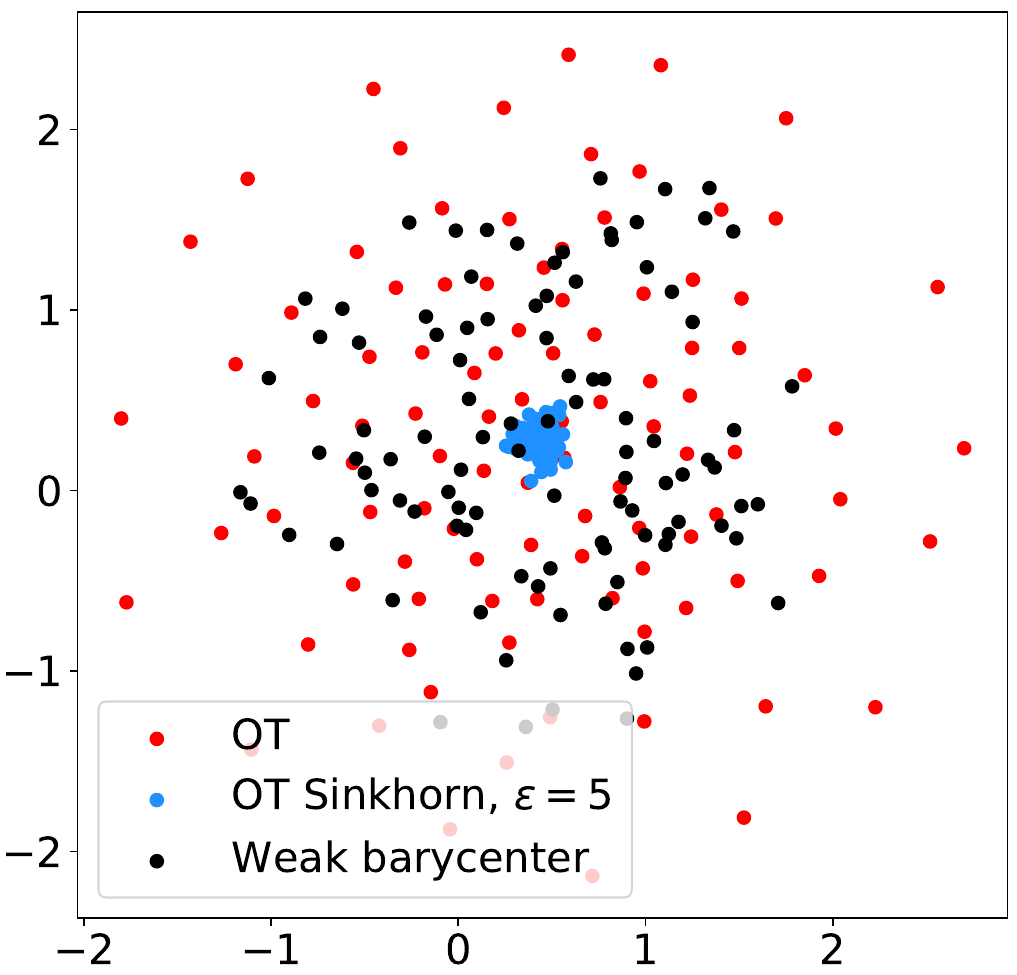} 
\end{tabular}
\caption{(left) Empirical Gaussian distributions and their OWT (black) and OT (red) barycenters computed with Algorithm \ref{algo:weakbarpop}. Illustration of the weak (black), OT (red) and OT Sinkhorn (blue) barycenters for different values of $\varepsilon=0.1,1,5$.}
\label{fig:gaussian}
\end{figure}

\begin{wrapfigure}[9]{r}{8cm}
\vspace{-1.2cm}\begin{tabular}{C{3.3cm} C{3cm}}
\includegraphics[scale=0.2]{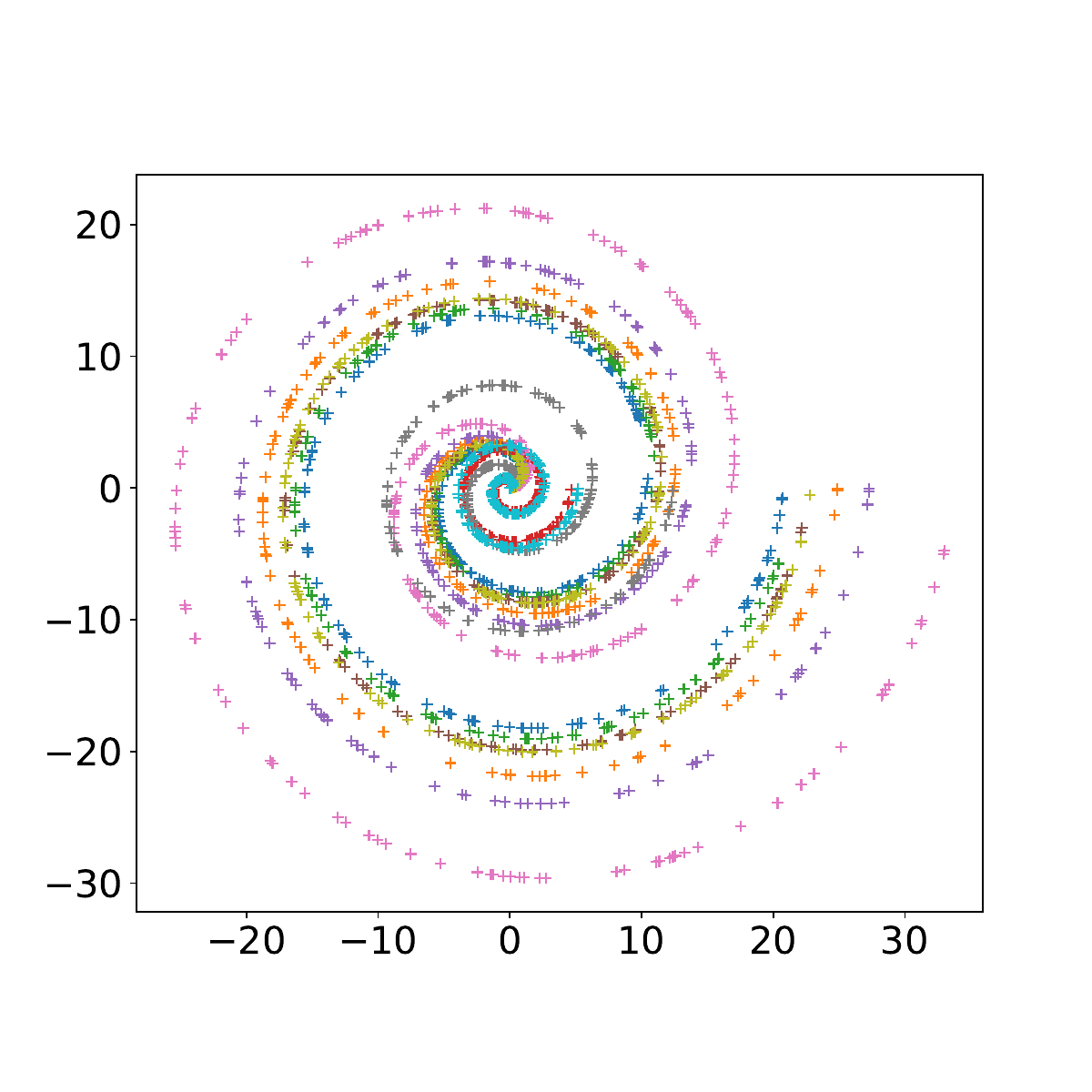} &
\includegraphics[scale=0.2]{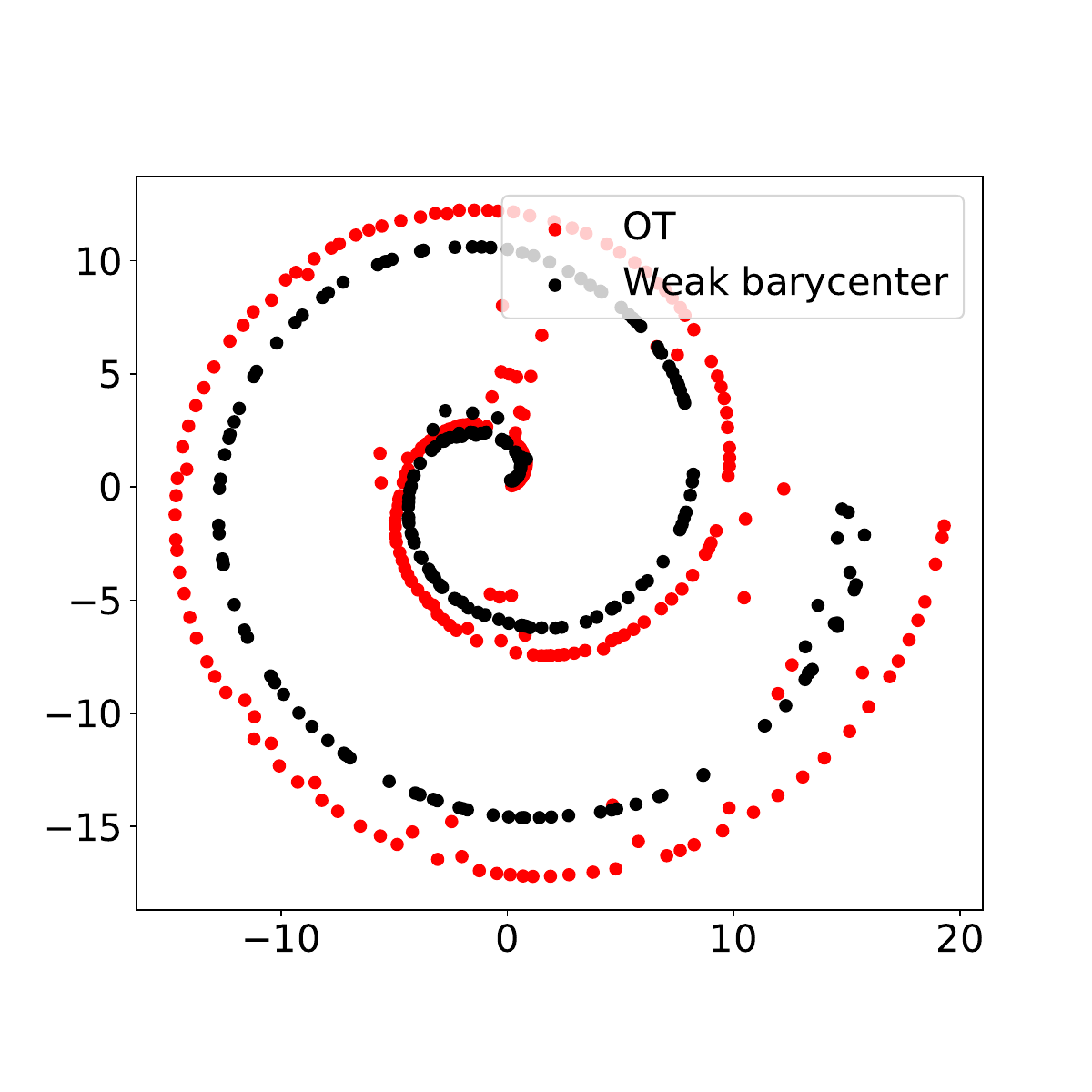} 
\end{tabular}
\vspace{-0.5cm}\caption{(left) Distributions supported on spiral. (right) OWT (black) and OT (red) barycenters computed with Algorithm  \ref{algo:weakbarpop}.}
\label{fig:spiral}
\end{wrapfigure}
\textbf{Spiral distributions.} ($r\in(200,225)$ \& $K=10$). In this experiment, we considered distributions supported on a spiral---see Fig.~\ref{fig:spiral} (left), with random ratio in $(0,3)$. The OT and OWT barycenters are presented in Fig.~\ref{fig:spiral} (right). Again, the weak barycenter seems to better preserve the shape of the spiral than the OT barycenter.
\vspace{1cm}

\subsection{Real-world dataset}
{\bf MNIST dataset.} We considered the well-known MNIST dataset \cite{lecun1998mnist} of grayscale images of handwritten digits. The images, of size $28\times 28$ pixels, can be normalised and thus be interpreted as discrete probability measures supported on a two-dimensional grid of size $28\times 28$. We computed the barycenters with $30$ steps of Algorithm \ref{algo:weakbar} between two digits "$8$", that are noisy versions of the same digit with the aim to produce a more stable barycenter. To produce noisy data, we randomly (Bernoulli $p = 0.1$) move pixels of the prototype digit displayed in Fig.~\ref{fig:mnist_8} (left). Fig.~\ref{fig:mnist_8} (right) shows the barycenters using the OWT, OT, and entropic-OT (for $\varepsilon=1$). This example illustrates how OWT reduces dispersion, so that weak barycenter provides the best uniformly spread results among the barycenters considered, with the two loops of the "$8$" well shaped.

\begin{figure}[h!]
\begin{center}
\begin{tabular}{C{1.9cm} C{1.9cm} C{2.2cm} | C{1.9cm} C{1.9cm} C{1.9cm}}
Prototype "$8$" & 1st noisy "$8$" & 2nd noisy "$8$" & OWT & OT & Regularised OT\\
\includegraphics[scale=0.26]{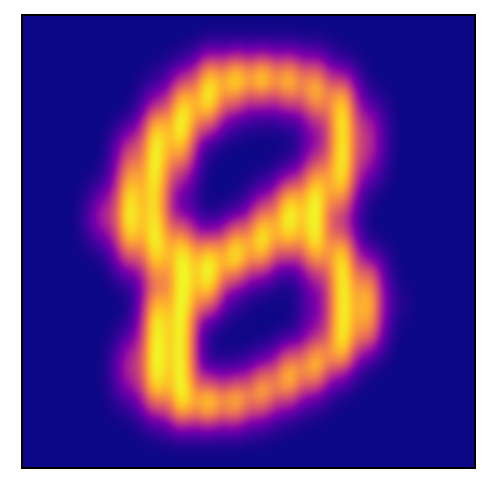} &
\includegraphics[scale=0.26]{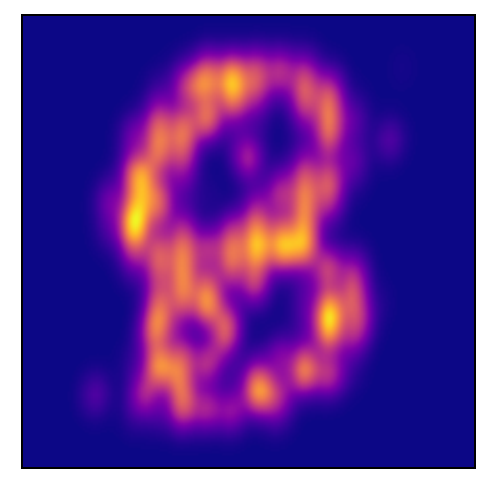} &
\includegraphics[scale=0.26]{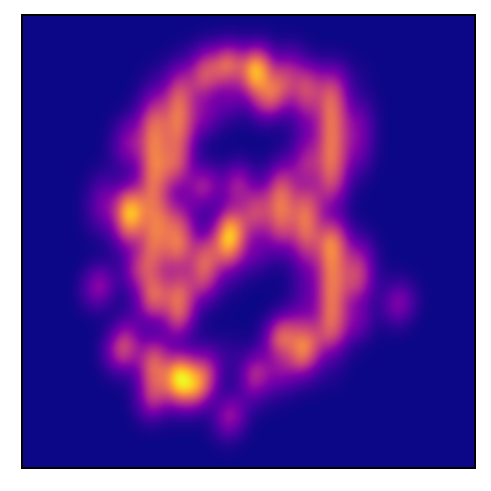} &
\includegraphics[scale=0.26]{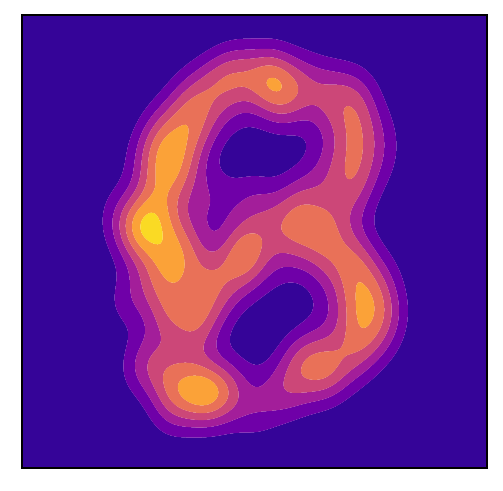}&
\includegraphics[scale=0.26]{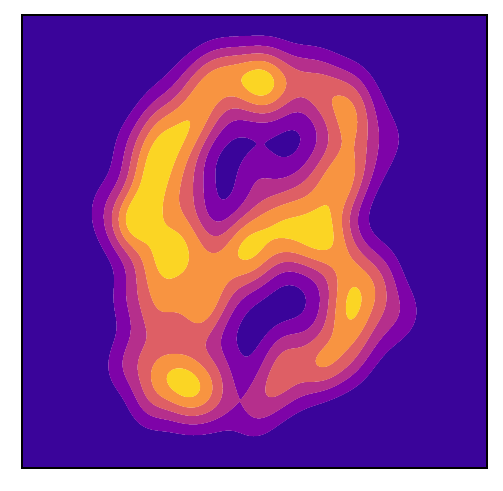}&
\includegraphics[scale=0.26]{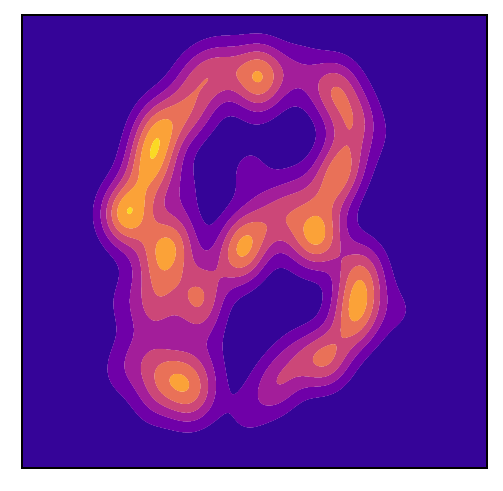}
\end{tabular}
\end{center}
\caption{Digit "$8$" (MNIST). From left to right: Prototype "$8$", first and second noisy versions of the prototype by randomly (Bernoulli $p = 0.1$) moving pixels, three barycenters constructed with Algorithm \ref{algo:weakbar} associted to the OWT plan, an OT plan and the entropy regularised OT plan for $\varepsilon= 1$.}
\label{fig:mnist_8}
\end{figure}

{\bf Cytometry dataset.} In biotechnology, \emph{flow cytometry} is measured through intracellular markers of single cells in a biological sample with the objective of recognising  common features across patients. However, these measurements are often disrupted by acquisition, rather than biological artefacts \cite{hahne2010per}, thus hindering the identification of common features. To address this challenge, we compute the weak barycenter for the forward-scattered light (FSC) and side-scattered light (SSC) cell's markers (using the flowStats package of Bioconductor \cite{gentleman2004bioconductor}). We considered $K=15$ patients and a variable  number of cells per patient between $88$ and $2185$. Fig.~\ref{fig:cytometry} shows the 15 distributions (left) and the computed barycenters (right), thus confirming the ability of the weak barycenter to resolve the alignment of the dataset, while maintaining the expected diamond-shape. Moreover, the advantage of our proposed streaming procedure is fully exploited in this setting, since data from one or several patient can arrive sequentially. Though this setting has been addressed with the Wasserstein barycenter in \cite{bigot2019data}, also in Fig.~\ref{fig:cytometry}, such method required a fixed grid to compute the barycenter unlike  our method, thus revealing the computational simplicity of the weak barycenter. 

\begin{figure}[ht!]
\centering
\begin{tabular}{C{4cm} C{4cm}}
\includegraphics[scale=0.2]{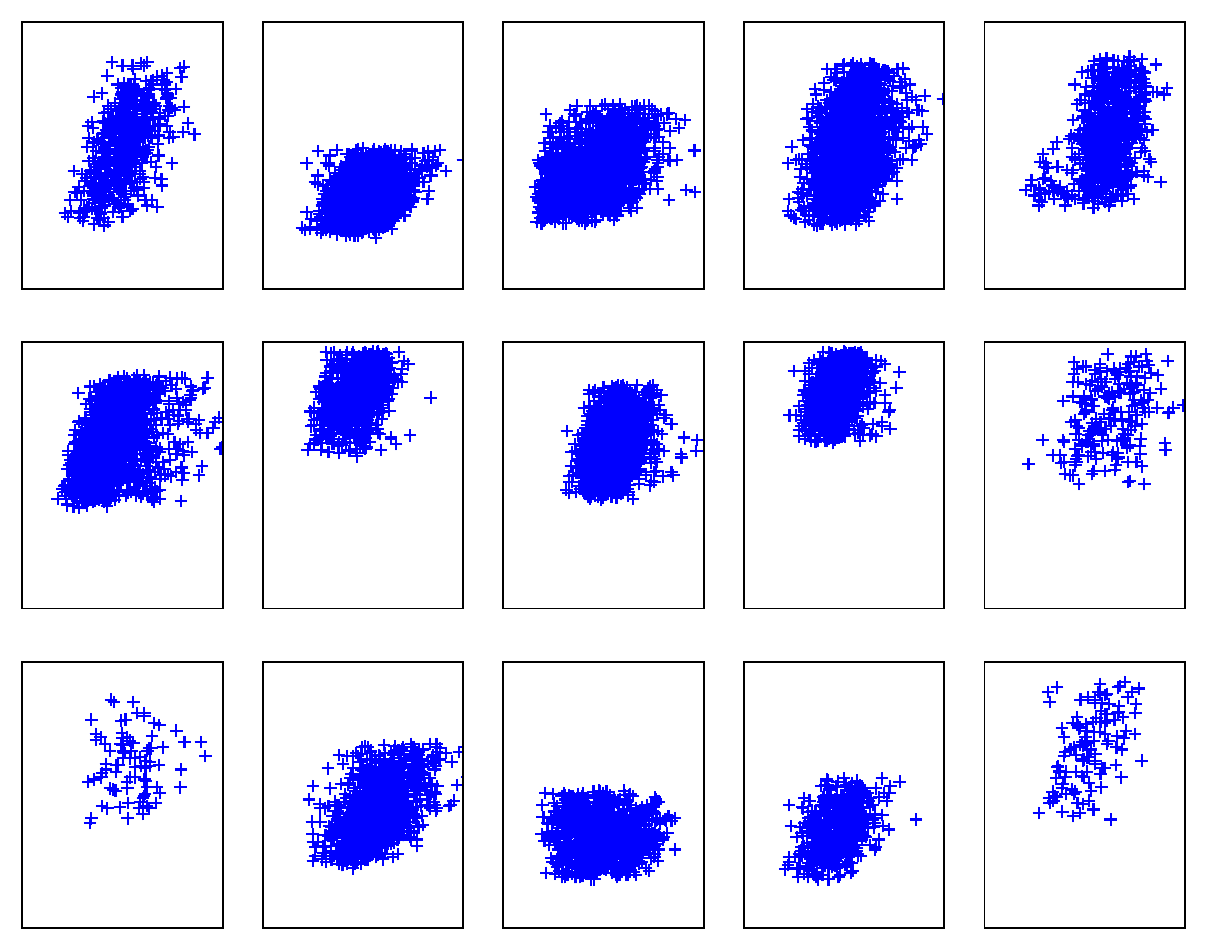} &
\includegraphics[scale=0.2,trim={0 0.4cm 0 0}, clip]{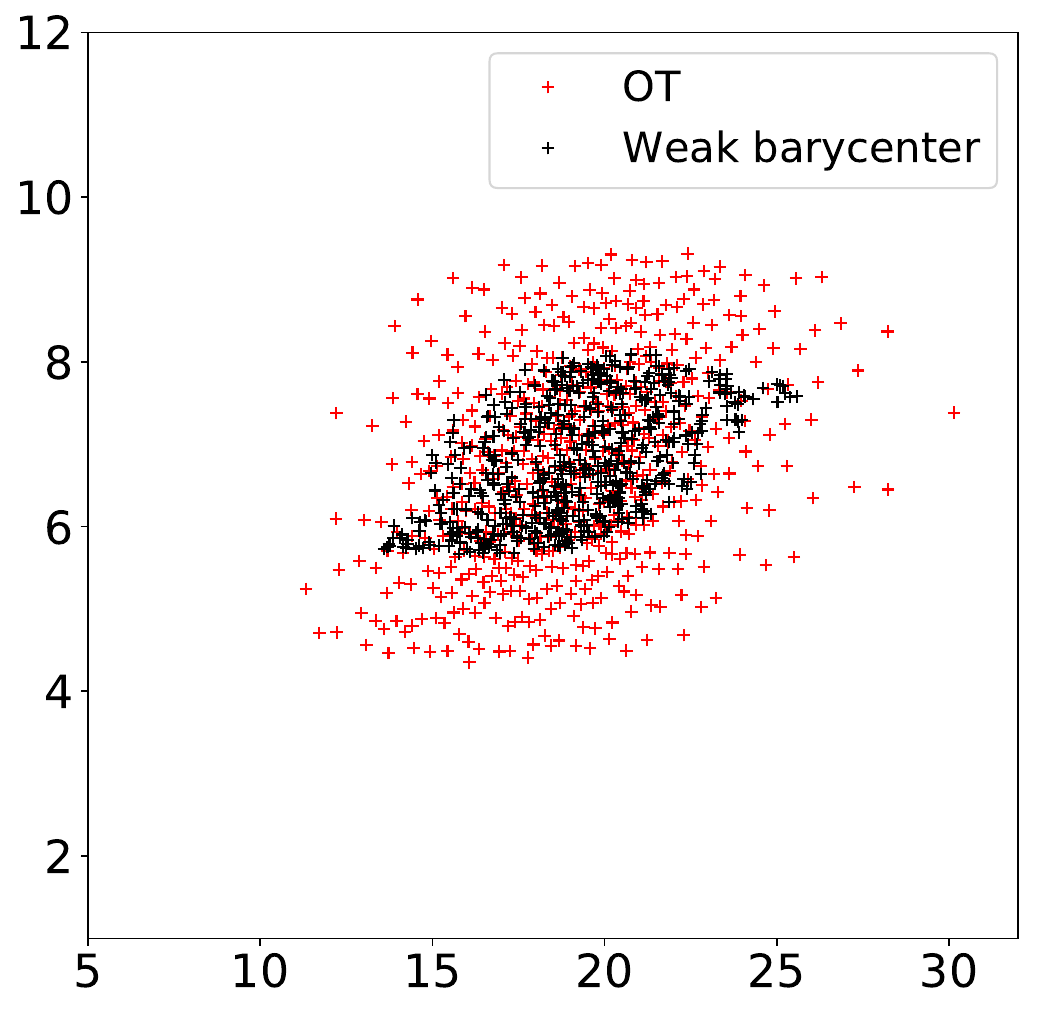}
\end{tabular}
\caption{(left) Cytometry dataset for $n=15$ patients and FSC vs. SSC cell's marker. (right) The weak-barycenter (black) computed with Algorithm \ref{algo:weakbarpop} and the OT barycenter (red). The data are represented with the same axis as the figure of barycenters.}
\label{fig:cytometry}
\end{figure}

%% file: chapters/CTF2021_appendix_final_version2.tex
\section{Additional mathematical background}
\subsection{$p$-Wasserstein distance}
\label{sec:Wass_dist}

 For $p=2$ and $\mu,\nu\in\PP_2(\R^d)$ such that $\mu$ is absolutely continuous (\emph{a.c.}) with respect to Lebesgue measure, the unique optimal plan is concentrated on the graph of a measurable map and Eq.~\eqref{def:OT} boils down to Monge's problem:
\begin{equation}
\label{def:Monge}
W_2(\mu,\nu) = \left(\underset{T\in\T(\mu,\nu)}{\min} \int_{\R^d\times\R^d}\Vert x-T(x)\Vert^2 \d\mu(x)\right)^{1/2},
\end{equation}
where $\T(\mu,\nu)$ is the set of measurable functions $T:\R^d\to\R^d$ such that $\nu=T\#\mu$. The \emph{pushforward} operator $\#$ is defined such that for any measurable set $B\subset\R^d$, we have $\nu(B)=\mu(T^{-1}(B))$. In such a case, the optimal measurable map $T$ in Eq.~\eqref{def:Monge} is uniquely defined (see \textit{e.g.} Th. 9.4 in \cite{villani2008optimal}) and called \emph{Monge map}.

\subsection{Continuity of  $V$}
\label{sec:continueV}
\begin{theorem}[\cite{backhoff2020weak}, Theorem 1.5]
\label{th:continuityV}
Let $(\mu_n)_n\subset\PP_2(\R^d)$ and $(\nu_n)_n\subset\PP_1(\R^d)$. Then
\[\left\{\begin{aligned}
& \mu_n \rightarrow \mu\quad\mbox{in} \ W_2\\
& \nu_n \rightarrow \nu\quad\mbox{in} \ W_1\\
\end{aligned}
\right.
\quad\Longrightarrow \quad\lim_n V(\mu_n | \nu_n) = V(\mu |\nu).
\]
\end{theorem}

\subsection{On the barycentric projection}
\label{sec:bar_projection}
For a given transport plan $\pi\in\Pi(\mu,\nu)$, with $\mu,\nu\in\PP_2(\R^d)$, the associated barycentric projection is given by $$S:x\mapsto \int_{\R^d}y\d\pi_x(y).$$
First, for each  $x\in \R^d, S(x)$ realises $\min_z \E_{Y \sim \pi_x} (\Vert z-Y\Vert^2)$. Second, this barycentric map $S$ is actually optimal for the Monge's problem Eq.~\eqref{def:Monge} between $\mu$ and $S\#\mu$, by Theorem \ref{th:1.4Julio}.

\begin{figure}[h!]
\centering
\begin{tabular}{cc}
\includegraphics[scale=0.32, trim={0 1.5cm 0 2cm}, clip]{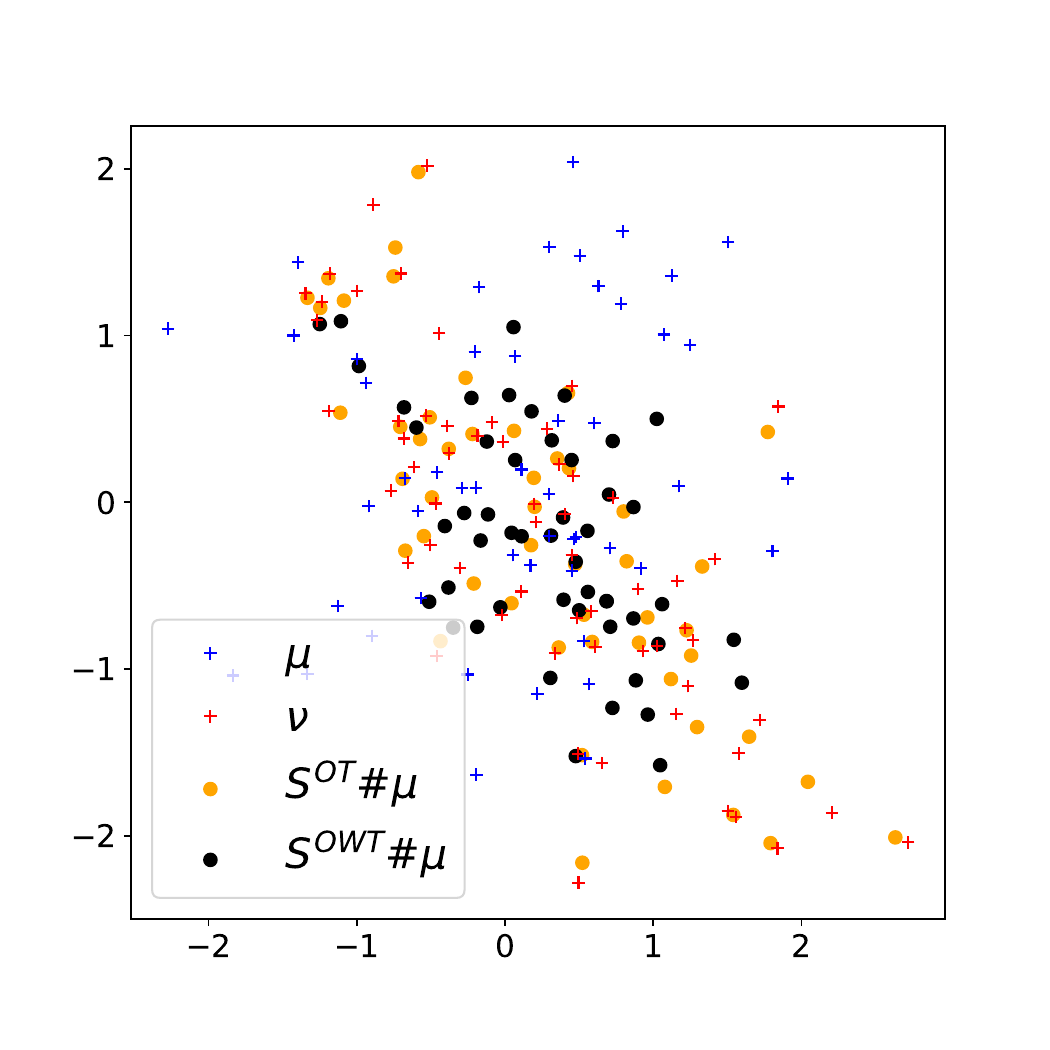} & \includegraphics[scale=0.32]{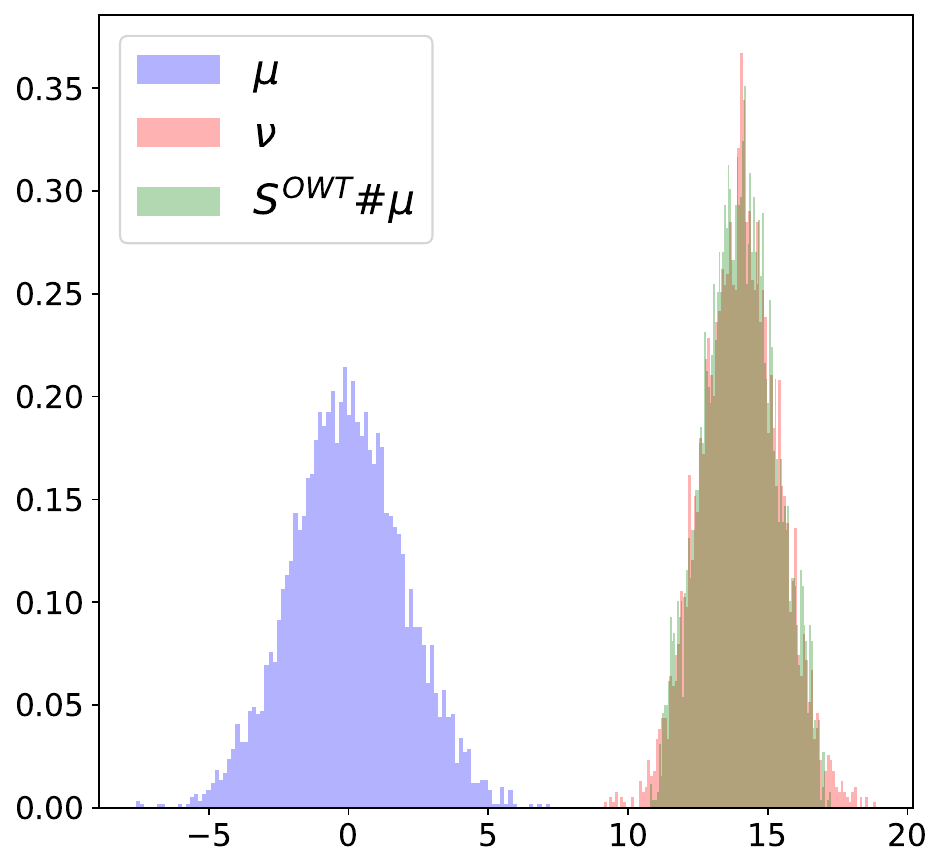}
\end{tabular}
\caption{Example of pushforward measures constructed from barycentric projections for two measures $\mu$ and $\nu$ in two dimensions (left) and one dimension (right).}
\label{fig:example_bary_proj}
\end{figure}
We next illustrate the differences between the optimal barycentric map and a barycentric map constructed from an OT plan in the classical Kantorovich formulation in Eq.~\eqref{def:OT}. We sampled $r=50$ observations $X_i$ and $m=60$ observations $Y_i$, each sets from a 2D Gaussian. We then defined the source and target distributions as $\mu = \frac{1}{r}\sum\delta_{X_i}$ and $\nu= \frac{1}{m}\sum\delta_{Y_i}$ respectively. Figure \ref{fig:example_bary_proj}(left),  shows these discrete distributions together with the pushforward measures $S^{OWT}\#\mu$ and $S^{OT}\#\mu$ constructed from an optimal weak plan $\pi^{OWT}$ and an optimal plan $\pi^{OT}$ respectively. The measure $S^{OT}\#\mu$ reasonably fits the target distribution $\nu$, since when $\mu$ is \emph{a.c.}, $S^{OT}\#\mu=\nu$. In particular, if $\mu$ and $\nu$ had the same number of points,  $S^{OT}\#\mu$ would have matched $\nu$. Regarding the measure $S^{OWT}\#\mu$, recall that $V(\mu | \nu) = \inf_{\eta\leq_c\nu} W_2^2(\mu,\eta) = W_2^2(\mu,S^{OWT}\#\mu)$, and therefore $S^{OWT}\#\mu\leq_c\nu$. Lastly, we have $W_2^2(\mu,\nu)=0.85$, and $V(\mu | \nu) = 0.52 \leq W_2^2(\mu,S^{OT}\#\mu)=0.81$ as expected.

In Figure \ref{fig:example_bary_proj}(right), we present an example in one dimension, where we sample $4000$ observations from $\NN(0,2)$ (resp. $\NN(14,1.4)$) to construct the empirical source measure $\mu$ (resp. empirical target measure $\nu$). The distributions $\mu$ and $\nu$ are presented in the form of histograms. The distribution resulting from an optimal weak transport map $S_{\mu}^{\nu}\#\mu$ is in convex order with $\nu$.

\section{Proofs of Section \ref{sec:weak_bar}}
\label{sec:proofs_weakbar}

\begin{proof}[Proof of Proposition \ref{prop:existence_weak_bar}]
Let $(\mu_m)_m\subset\PP_2(\R^d)$ be a minimising sequence of $F(\mu):=\sum_{i=1}^n\lambda_i V(\mu | \nu_i)$ and let $M<\infty$ be such that $F(\mu_m) \leq M$ for all $m$. Then $(\mu_m)_m$ is tight. Indeed, 
\begin{align*}
\int \Vert x\Vert^2\d\mu_m(x) & \leq 2\sum_{i=1}^n \lambda_i \inf_{\pi\in\Pi(\mu_m,\nu_i)} \left[\int\Vert x-\int y \d\pi_x(y)\Vert^2 \d\mu_m(x)+\int \Vert \int y \d\pi_x(y)\Vert^2 \d\mu_m(x)\right]\\
& \leq  2M + 2\iint\Vert y \Vert^2 \d\pi_{x}(y)\d\mu_m(x) \leq 2M + 2\sum_{i=1}^n \lambda_i \int \Vert y \Vert^2\d\nu_i(y),
\end{align*}
where the second inequality  comes from Jensen's inequality.
By Prokhorov's theorem, there exists a subsequence still denoted $(\mu_m)_m$ that weakly converges toward a probability measure $\mu^* $. Recall that $\mu$ belongs to $\PP_2(\R^d)$ since $\Vert\cdot\Vert^2$ is a l.s.c. function bounded from below and therefore $\int \Vert x\Vert^2 \d\mu(x)\leq \liminf_{m}\int \Vert x\Vert^2 \d\mu_m(x)<\infty$. By Theorem \ref{th:continuityV}, we have that
\[ F(\mu^*)=\sum_{i=1}^n \lambda_i \lim_m V(\mu_m | \nu_i) =  \lim_m F(\mu_m) = \min_{\mu\in \PP_2(\R^d)} F(\mu), \]
thus $F$ admits at least a minimiser.
\end{proof}

\begin{proof}[Proof of Lemma \ref{lemma:diracweak}]
By Strassen's theorem, we can build  $X'\sim \mu'$ and $X\sim \mu$   in the same probability space, in such a way that $\E(X\vert X')= X'$.   
Denote by 
$\eta^*$ the law $\eta$ attaining $ \inf_{\eta \leq_c \nu} W_2^2(\mu,\eta) $, and  let  $(X,Z)=(X, S_{\mu}^{\nu}(X))  $  be the realisation of the optimal  coupling  for $W_2$  of $ \mu $ and $ \eta^*$, which can also be constructed in the same probability space due to its specific form.  Then, by (the conditional version of) Jensen's inequality we have
$$
V(\mu |\nu)  =  W_2^2(\mu,\eta^*) = \E \left[  \E \left( \Vert X - S_{\mu}^{\nu}(X) \Vert^2 \vert X' \right)  \right] \geq \E \Vert  X'  - \E(  S_{\mu}^{\nu}(X)\vert X')  \Vert^2. 
$$
Recall  now that $S_{\mu}^{\nu}(X)= \E(Y\vert X)$, where the conditional expectation is a measurable function only of $X$,  constructed  from the joint law $\pi^{\mu,\nu}$. Thus,  for every nonnegative convex function $\phi$,  by applying twice Jensen's inequality we get
$$\E \phi( \E(  S_{\mu}^{\nu}(X)\vert X') ) \leq \E \phi(  S_{\mu}^{\nu}(X) ) =  \E  \phi( \E(Y\vert X) ) \leq \E \phi (Y), $$
where $Y\sim \nu$.  That is to say, the law $\eta$ of the r.v.  $ \E(  S_{\mu}^{\nu}(X)\vert X') $ satisfies $\eta \leq_c \nu$. 
It follows that
$$ V(\mu |\nu)  \geq  W_2^2(\mu' ,\eta)  \geq \inf_{\tilde{\eta} \leq_c \nu}W_2^2(\mu',\tilde{\eta})= V(\mu' | \nu).$$
This immediately implies that $\mu'$ is a weak barycenter whenever $\mu$ is.  In particular,  if $\mu$ is a weak barycenter, then so is the Dirac mass supported on its mean. We then deduce that the set of minimisers of $\sum_{i} \lambda_i V(\mu | \nu_i)$ admits at least a Dirac  mass $\delta_{\omega}$ and
$$V(\delta_{\omega} | \nu_i)=\int \Vert x-\int y\d\pi_{x}(y)\Vert^2\d\delta_{\omega}(x) = \Vert {\omega} - \E Y_i\Vert^2.$$
This implies that $\inf_{\omega}\sum\lambda_i V(\delta_{\omega} | \nu_i)$ is uniquely attained for $\bar{\omega}=\sum\lambda_i \E Y_i $. 
\end{proof}

\begin{proof}[Proof of Proposition \ref{prop:bary_charact}]
A probability measure $\mu$ is a weak barycenter if and only if $\sum_{i=1}^n \lambda_i V(\mu | \nu_i)$ is equal to the r.h.s. of \eqref{eq:obj_function}. Let us suppose first that $\E Y_i = m$ for all $1\leq i\leq n$, in which case the infimum in \eqref{eq:weak_barycenter} is equal to $0$. Then, $\mu$ is a weak barycenter if and only if $\mu\leq_c \nu_i$ for all $1\leq i\leq n$ by definition of weak optimal transport \eqref{def:weak_transport}, since in this case $V(\mu | \nu_i)=0$. The general case can be reduced to the previous one, noting that 
\begin{align*} 
V(\mu |\nu_i)  & = \inf_{\eta \leq_c \nu_i} W_2^2(\mu,\eta) \\
 & = \inf_{\eta \leq_c \hat{\nu}_i} W_2^2(\hat{\mu},\eta) + \| \E_{\mu}(X)- \E_{\nu_i}(Y_i) \|^2 \\
 & = V(\hat{\mu } |\hat{\nu}_i)+   \| \E_{\mu}(X)- \E_{\nu_i}(Y_i) \|^2,  
 \end{align*}
 so that minimising  $ \sum_{i=1}^n \lambda_i V(\mu | \nu_i) $  over $\mu\in {\cal P}(\R^d) $ is equivalent to minimising $ \sum_{i=1}^n \lambda_i  V(\mu'  | \hat{\nu}_i)  + \sum_{i=1}^n \lambda_i  \Vert {\omega} - \E_{\nu_i} Y_i\Vert^2  $ over the  two independent parameters  $(\omega,\mu')$, with $\omega \in \R^d$  and $\mu'\in {\cal P}(\R^d) $ centered, taking $\mu$ as the law of $X=X'+ \omega$ with $X'\sim \mu'$. 
\end{proof}

\begin{proof}[Proof of Lemma \ref{lemma:weakVSwass}]
Thanks to Prop. 3.3 in \cite{alvarez2016fixed}, we have that the (unique) Wasserstein barycenter verifies $\tilde{\mu}= \left(\sum_{i=1}^n \lambda_i T^{\nu_i}_{\tilde{\mu}}\right)\#\tilde{\mu}$ where $T^{\nu_i}_{\tilde{\mu}}$ is the optimal Monge map between $\tilde{\mu}$ and $\nu_i$ (see \eqref{def:Monge}). Moreover, from Proposition \ref{prop:fixedpoint}, a weak barycenter $\bar{\mu}$ also checks $\bar{\mu} = \left(\sum_{i=1}^n \lambda_i S^{\nu_i}_{\bar{\mu}}\right)\# \bar{\mu}$, where $S^{\nu_i}_{\bar{\mu}}$ is the optimal barycentric projection associated to $\bar{\pi}^i$ for $V(\bar{\mu} | \nu_i)$. Therefore, by Jensen's inequality applied twice,
\begin{align*}
W_2^2&(\bar{\mu}, \tilde{\mu}) \leq \iint \Vert x-y\Vert^2 \d\bar{\mu}(x)\d\tilde{\mu}(y)
= \iint \Vert \sum_{i=1}^n \lambda_i S^i_{\bar{\mu}}(x)-\sum_{i=1}^n \lambda_i T^{i}_{\tilde{\mu}}(y)\Vert^2 \d\bar{\mu}(x)\d\tilde{\mu}(y)\\
& \leq \sum_{i=1}^n \lambda_i\iiint \Vert T^{i}_{\tilde{\mu}}(y) - z\Vert^2 \d\bar{\pi}^i_x(z)\d\bar{\mu}(x)\d\tilde{\mu}(y)
 = \sum_{i=1}^n\lambda_i \iint \Vert T^{i}_{\tilde{\mu}}(y) - z\Vert^2 \d\tilde{\mu}(y) \d\bar{\pi}^i(x,z)\\
 &= \sum_{i=1}^n\lambda_i \iint \Vert y - z\Vert^2 \d\nu_i(y) \d\nu_i(z)
 =2\sum_{i=1}^n\lambda_i \iint \left(\E \Vert Y_i\Vert^2-\Vert\E Y_i\Vert^2\right).
\end{align*}
\end{proof}


\begin{proof}[Proof of Theorem \ref{theo:latent}]
  Observe first that, by Theorem \ref{th:1.4Julio}  and Strassen's theorem, solving the OWT problem \eqref{def:weak_transport}  provides a  unique (in law) coupling of three random variables $(X,Y,Z)$ such that:
\vspace{-0.7em}\begin{itemize}
\setlength\itemsep{0em}
\item[i)]  $(X,Y) $ has joint law $\pi^{\mu,\nu}$; in particular $X$ and $Y$ have the laws $\mu$ and $\nu$ respectively, 
\item[ii)] $Z= S_{\mu}^{\nu}(X) = \E( Y\vert X)$ a.s., it has law $\eta^*$ and  it is optimally coupled to $X$ in the sense of the optimal transport problem \eqref{def:OT},
\item[iii)]  $(Z,Y)$ is a martingale, that is $\E( Y\vert Z)=Z$ a.s.. 
\end{itemize}
\vspace{-0.7em}Bringing all together we get the decomposition: 
\begin{equation}\label{eq:YZYZ}
Y= Z + Y -Z = S_{\mu}^{\nu}(X) + Y- \E(Y\vert  X ). 
\end{equation}
Now, by Lemma \ref{lemma:diracweak}, if $X\sim \mu$ then the Dirac mass $\delta_{\E X}$ is a weak barycenter too. Thus we have on one hand: 
\begin{equation}\label{eq:barybardelta}  \sum_{i=1}^n \lambda_i V(\mu | \nu_i) = \sum_{i=1}^n \lambda_i V( \delta_{\E X} | \nu_i).
\end{equation}
 Using  Jensen's inequality, we see on the other hand that
\begin{align*}
V(\mu |\nu_i) & = \inf_{\eta \leq_c \nu_i} W_2^2(\mu,\eta)\\
& = \inf_{\eta \leq_c \nu_i} \E \Vert X - Z\Vert^2, \mbox{ with} \  (X,Z) \mbox{ an optimal  coupling  for }W_2^2 \mbox{ of } \mu \mbox{ and }  \eta \\
& \geq \inf_{\eta \leq_c \nu_i} \Vert \E X - \E Z\Vert^2 = \inf_{\eta \leq_c \nu_i} W_2^2(\delta_{\E X},\delta_{\E Z })\\
& \geq \inf_{\tilde{\eta} \leq_c \nu_i}W_2^2(\delta_{\E X},\tilde{\eta})= V(\delta_{\E X} | \nu_i).
\end{align*}
Identity \eqref{eq:barybardelta}  thus implies  $V(\mu | \nu_i) = V( \delta_{\E X} | \nu_i)$  for all $i$.  Denoting by $\eta_i$ the law $\eta$ attaining $ \inf_{\eta \leq_c \nu_i} W_2^2(\mu,\eta) $, and by $(X,Z_i)$  the optimal  coupling  for $W_2$ of $ \mu $ and $ \eta_i$,  we see that the latter can only occur if the equality case  $ \E \Vert X - Z_i\Vert^2  =   \Vert \E X - \E Z_i\Vert^2$   in Jensen's inequality holds. 
This implies that $X - Z_i$ is deterministic for each $i$. Since $\E Z_i = \E Y_i$, we  thus must have $X - Z_i= \E X - \E Y_i$.  Taking $Z=Z_i$ and $Y=Y_i$ in Eq.~\eqref{eq:YZYZ}, and noting that $ S_{\mu}^{\nu_i}(X) =Z_i = X - (  \E X - \E Y_i )$ the statement follows. 
\end{proof} 

%
%
%


\section{Proofs of Section \ref{sec:pop_barycenter}}
\label{sec:proofs_popweakbar}
\begin{proof}[Proof of Lemma \ref{lemma:Smeasurable}]


\CB{We will first show that the set-valued mapping (or correspondence) associating  with a pair $(\mu,\nu)\in(\PP_2(\R^d))^2$  the set of plans $\pi\in\Pi^*(\mu,\nu)$ attaining $V(\mu|\nu)$ in \eqref{def:weak_transport} has a measurable selection. That is, there exists a measurable function associating to each such pair $(\mu,\nu)$ a (single) optimal plan denoted $\pi^{\mu,\nu}\in\Pi^*(\mu,\nu)$. Since $\PP_2(\R^{2d})$ is a Polish space, by the Kuratowski–Ryll-Nardzewski Selection Theorem, stated in Theorem 18.13 in  \cite{charalambos2013infinite}, and by Lemma 18.3 therein, it is enough to show that
\begin{itemize}
\item[a)] $\Pi^*(\mu,\nu)$ is a closed subset of $\PP_2(\R^{2d})$ for all $(\mu,\nu)\in(\PP_2(\R^d))^2$, and
\item[b)] for each closed set $F\subseteq\PP_2(\R^{2d})$, the set $\{(\mu,\nu)\in(\PP_2(\R^d))^2 \ :\  \Pi^{\ast}(\mu,\nu)\cap F\neq \varnothing\}$ is closed in $(\PP_2(\R^d))^2$.
\end{itemize}
a) Let $\{\pi^k\}_{k\in\mathbb{N}} \subseteq \Pi^*(\mu,\nu)$ be a sequence converging to $\pi$ w.r.t. $W_2$. By Proposition 2.8 in \cite{backhoff2019existence},
\[ \liminf_k V(\mu_k | \nu_k) =\liminf_k \int \Vert x-\int y\d\pi^{k}_x\Vert^2\d\mu(x)\ \geq \int \Vert x-\int y\d\pi_x\Vert^2\d\mu(x),\]
hence $\pi\in\Pi^*(\mu,\nu)$ and this set is closed.

b) Let $F\subseteq\PP_2(\R^{2d})$ be closed, and $\{(\mu_k,\nu_k)\}_{k\in\mathbb{N}}$ be a sequence in $ (\PP_2(\R^d))^2$ such that $\Pi^*(\mu_k,\nu_k)\cap F\neq \varnothing$ and $(\mu_k)_k$ (resp. $(\nu_k)_k$) converges towards $\mu$ (resp. $\nu$) w.r.t. $W_2$. Then if we take $\pi^k\in\Pi^*(\mu_k,\nu_k)\cap F$ for each $k$, the sequence $\{\pi^k\}_k$ is tight and has a weakly convergent subsequence denoted $\{\pi^{k'}\}_{k'}$, converging towards some $\pi\in\Pi(\mu,\nu)$. By Proposition 2.8 in \cite{backhoff2019existence}, we have
\[ \liminf_{k'} V(\mu_{k'} | \nu_{k'}) =\liminf_{k'} \int \Vert x-\int y\d\pi^{k'}_x\Vert^2\d\mu_{k'}(x)\ \geq \int \Vert x-\int y\d\pi_x\Vert^2\d\mu(x) \geq V(\mu | \nu).\]
However we have $\lim_{k'} V(\mu_{k'} | \nu_{k'}) = V(\mu,\nu)$ thanks to Theorem \ref{th:continuityV}, hence $\pi\in\Pi^*(\mu,\nu)$. Moreover, since
\[\int \Vert x\Vert^2+\Vert y\Vert^2 \, \d\pi^{k'}(x,y)=\int\Vert x\Vert^2\d\mu_{k'}(x)+\int\Vert y\Vert^2\d\nu_{k'}(y) \underset{k' \to \infty}{\longrightarrow} \int\Vert x\Vert^2\d\mu(x)+\int\Vert y\Vert^2\d\nu(y),\]
we have that $W_2(\pi^{k'},\pi)\underset{k' \to \infty}{\longrightarrow} 0$. Since $\pi^{k'}\in F$ for all $k'$, we deduce that $\pi\in F\cap\Pi^*(\mu,\nu)\neq \varnothing$ as required.}

\CB{ From now on, for each $(\mu,\nu)\in\PP_2(\R^{2d})$, we denote by $\pi^{\mu,\nu}\in\Pi^*(\mu,\nu)$ an element such that $(\mu,\nu)\mapsto\pi^{\mu,\nu}$ is measurable.
We next establish the joint measurability of  $(x, \nu)\in  \R^d \times \PP_2(\R^d)\mapsto  S_{\mu}^{\nu}(x)$  for fixed $\mu$. Notice this is a stronger statement than just measurability  in the $x$ variable, for each $(\mu,\nu)$. Write $\bar{B}(x,r)$   for the closed ball of radius $r>0$ centered at $x$. One easily  checks that   the function 
\[  (x,\pi)\mapsto  \Psi_r(x,\pi):=  \frac{ \int y \mathbf{1}_{\{ (y,z) : z\in \bar{B}(x,r)  \} }\d\pi(z,y)}{ \int \mathbf{1}_ {\{ (y,z) : z\in \bar{B}(x,r)\} } \d\pi(z,y) }  \]
is measurable w.r.t. the pair $(x,\pi)\in\R^d\times\Pi(\mu,\nu)$, the two integrals being  limits of integrals   with respect to $\d\pi(z,y)$, of some bounded continuous functions of $(x,y,z)$.  Thus, $\limsup_{r\to 0}   \Psi_r(x,\pi) $,  $\liminf_{r\to 0}   \Psi_r(x,\pi) $ and  the function  $ \Phi (x,\pi) : =  \limsup_{r\to 0}   \Psi_r(x,\pi)  \mathbf{1}_{\{ \limsup_{r\to 0}   \Psi_r(x,\pi) =\liminf_{r\to 0}   \Psi_r(x,\pi) \}}$  depend in a measurable way on  $(x,\pi)$.  It follows that  $(x,\mu,\nu)\mapsto  \Phi (x,\pi^{\mu,\nu})$ is  measurable as the composition of two measurable functions. But notice that for each  fixed $\mu \in  \PP_2(\R^d)$ one has  $ \Psi_r(x,\pi^{\mu,\nu} ) = \frac{ \int_{ \bar{B}(x,r)}\left[ \int  y \d\pi^{\mu,\nu}_z(y)  \right] \d \mu(z) }{ \mu (\bar{B}(x,r))}   $ which,  by the Lebesgue derivation theorem for Radon measures (\emph{see e.g.} \cite{botsko2003elementary}), converges $\d \mu$ a.s.  in $x$,  to $\int  y \d\pi^{\mu,\nu}_x(y)$. By Theorem \ref{th:1.4Julio}  this measurable function of $x$ does not depend on the particular selection of an element  $\pi^{\mu,\nu}\in\Pi^*(\mu,\nu) $ and is $ \d \mu(x)  $ a.e. equal to  $ S_{\mu}^{\nu}(x)$.  Thus,  for each $\mu \in  \PP_2(\R^d)$, 
 \[  S_{\mu}^{\nu}(x) = \Phi (x,\pi^{\mu,\nu})   \quad \mbox{for all } \nu\in  \PP_2(\R^d) \mbox{ and } \d \mu   \mbox{ a.s. in } x, \]
 with $(x,\nu) \mapsto  \Phi (x,\pi^{\mu,\nu})$ a measurable function. The conclusion follows.  }
 
\end{proof}

\begin{proof}[Proof of Proposition \ref{prop:existence_pop_bar}]
By Theorem 6.16 in \cite{villani2008optimal}, we know that their exists a sequence of discretely supported distributions $(\Q_n)_n\subset\PP_2(\PP_2(\R^d))$ of the form $\Q_n=\sum_{i=1}^n\lambda_i\delta_{\nu_i}$, with $(\lambda_i)_{1\leq i\leq n}$ in the simplex, and such that $W_2^2(\Q,\Q_n):=\inf_{\pi\in\Pi(\Q,\Q_n)}\int W_2^2(\nu,\tilde{\nu})\d\pi(\nu,\tilde{\nu})\rightarrow 0$. We set
$$L_n(\mu):=\int_{\PP_2(\R^d)}V(\mu|\nu)\d\Q_n(\nu)=\sum_{i=1}^n\lambda_i V(\mu | \nu_i).$$
We denote $\mu^n\in\PP_2(\R^d)$ the minimiser of $L_n$. Let us prove that $(\mu^n)_n$ is tight. First, $\mu^n$ admits moments of order $2$ thanks to Jensen's inequality:
\begin{align*}
\int\Vert x\Vert^2 & \d\mu^n(x)\leq \sum_{i=1}^n\lambda_i\left[\int\Vert x-S_{\mu^n}^{\nu_i}(x) \Vert^2\d\mu^n(x)+\int\Vert S_{\mu^n}^{\nu_i}(x)\Vert^2\d\mu^n(x)\right]\\
&\leq \sum_{i=1}^n\lambda_i V(\mu^n | \nu_i)+\sum_{i=1}^n\lambda_i\int \Vert y\Vert^2\d\nu_i(y)\\
& \leq  \sum_{i=1}^n\lambda_i V(\mu | \nu_i)+\sum_{i=1}^n\lambda_i\int \Vert y\Vert^2\d\nu_i(y) \quad\mbox{for some $\mu\in\PP_2(\R^d)$ since $\mu^n$ minimises $L_n$}\\
&\leq 2\int \Vert x\Vert^2\d\mu(x)+3\sum_{i=1}^n\lambda_i\int \Vert y\Vert^2\d\nu_i(y),
\end{align*}
where the last inequality comes from $V(\mu | \nu_i)=\int \Vert x-S_{\mu}^{\nu_i}(x)\Vert^2\d\mu(x)\leq 2\int \Vert x\Vert^2\d\mu(x) +2\int \Vert S_{\mu}^{\nu_i}(x)\Vert^2\d\mu(x)$. Moreover, since $W_2^2(\Q,\Q_n)\rightarrow 0$, we have (Lemma 5.1.7 in \cite{ambrosio2008gradient}) that $\int\psi(\nu)\d\Q_n(\nu)\rightarrow\int\psi(\nu)\d\Q(\nu)$ for any function $\psi$ such that $\vert \psi(\nu)\vert\leq a + bW_2^2(\nu,\nu_0), a,b\geq 0$. In particular, choosing $\psi(\nu)=W_2^2(\nu,\delta_0)=\int \Vert y\Vert^2\d\nu(y)$, it implies that $\sum_{i=1}^n\lambda_i\int \Vert y\Vert^2\d\nu_i(y)\rightarrow \int \int \Vert y\Vert^2\d\nu(y)\d\Q(\nu)<\infty$. Therefore $\left(\sum_{i=1}^n\lambda_i\int \Vert y\Vert^2\d\nu_i(y)\right)_n$ is bounded and $(\mu^n)_n$ is tight. Thus by Prokhorov's theorem, there exists a subsequence, still denoted $(\mu^n)_n$, that converges towards $\bar{\mu}$.

Let us now prove that this particular $\bar{\mu}$ minimises the function $L:\mu\mapsto \int_{\PP_2(\R^d)}V(\mu|\nu)\d\Q(\nu)$. First, let $\eta\in\PP_2(\R^d)$,  still by Lemma 5.1.7 in \cite{ambrosio2008gradient} and since $V(\eta | \nu)\leq W_2^2(\eta , \nu)$, we get that $L(\eta)=\int V(\eta | \nu)\d\Q(\nu)\geq \liminf_{n\to\infty} \int V(\eta | \nu)\d\Q_n(\nu)$. Since for each $n$, the distribution $\mu^n$ minimises $L_n$, we have
\begin{equation}
\label{eq:liminf_existence}
\liminf_{n\rightarrow\infty} \int V(\eta | \nu)\d\Q_n(\nu)\geq \liminf_{n\to\infty} \int V(\mu^n | \nu)\d\Q_n(\nu).
\end{equation}
Thanks to Fatou's Lemma for sequences of measures $(\Q)_n$ (see \cite{feinberg2020fatou}), we have that
\[\liminf_{n\to\infty} \int V(\mu^n | \nu)\d\Q_n(\nu) \geq  \int\liminf_{n\to\infty}  V(\mu^n | \nu)\d\Q(\nu)= \int V(\bar{\mu} | \nu)\d\Q(\nu),\]
where the last equality comes from the lower semi-continuity of $V$ (Theorem 2.9 in \cite{backhoff2019existence}). This proves that $\bar{\mu}$ minimises $L$. 
\end{proof}

\section{Proofs of Section \ref{sec:algorithms}}
\label{sec:proof_algorithms}

The proof of Theorem \ref{th:continuityG}, on the continuity of $G:\mu\mapsto \left(\sum_{i=1}^n \lambda_i S_{\mu}^{\nu_i}\right)\# \mu$, leans on the two following technical lemmas.

\begin{lemma}
\label{lemma:UI2moments}
Let  $(\rho_m)_m$  be a given sequence and  $\nu$ be a fixed  law in $ \PP_2(\R^d)$.  For each $m$, let  $S_m:= S_{\rho_m}^{\nu}$ denote the barycenter map associated with an  optimal coupling $\pi^{\rho_m,\nu}$ for \eqref{def:weak_transport}. Then, the sequence of laws $(S_m \# \rho_m)_m $ has uniformly integrable second moments.  
\end{lemma}
\begin{proof}[Proof of Lemma \ref{lemma:UI2moments}]
Let  $(X_m,Y_m)$ be  a pair of  random variables (r.v.) with joint law $\pi^{\mu_m,\nu}$,  defined on some probability space  $(\Omega, {\cal F},\P)$. Notice that $S_m \# \rho_m$ is the law of the  r.v. $\E(Y_m\vert X_m)$.   Then, for each  $M , K\geq 0$  we   have
\begin{align*}
 \int_{\{\Vert x\Vert^2\geq M\}}  \Vert x\Vert^2 \d S_m\#\rho_m(x) 
  =  & \,  \E( \|  \E(Y_m\vert X_m) \|^2 \mathbf{1}_{ \{ \| \E(Y_m\vert X_m) \|^2 \geq M}\}  )\\
  \leq  & \,  \E(  \E(\| Y_m \|^2 \vert X_m)  \mathbf{1}_{ \{ \| \E(Y_m\vert X_m) \|^2 \geq M\} } ) \\
=  & \,  \E(   \| Y_m \|^2  \mathbf{1}_{   \{  \| \E(Y_m\vert X_m) \|^2 \geq M,   \| Y_m \|^2 \geq K \}} ) \\ & +   \E(   \| Y_m \|^2  \mathbf{1}_{  \{  \| \E(Y_m\vert X_m) \|^2 \geq M,   \| Y_m \|^2 < K \}} ) \\
\leq & \,   \E(   \| Y_m \|^2  \mathbf{1}_{  \| Y_m \|^2 \geq K \}} ) + \frac{K}{M}  \E( \|  \E(Y_m\vert X_m) \|^2 ),
\end{align*}
where we  have used Jensen's inequality and the fact that $ \E(Y_m\vert X_m) $ is measurable w.r.t. the $\sigma$-field  generated by $X_m$. Applying Jensen's inequality to the last term again, and recalling that $Y_m$ has law $\nu\in  \PP_2(\R^d)$, we deduce that
\begin{equation}\label{boundUI2mom}
\begin{split}
\sup_m  \int_{\{\Vert x\Vert^2\geq M\}}  \Vert x\Vert^2 \d S_m\#\rho_m(x)
\leq & \,  \int_{\{\Vert x\Vert^2\geq K\}}  \Vert x\Vert^2 \d \nu (x)   + \frac{K}{M}  \int  \Vert x\Vert^2 \d \nu (x), 
 \end{split}
 \end{equation}
which is smaller than a given $\varepsilon>0$, by choosing  $K>0$ and then $M>0$ large enough.
\end{proof}

\begin{lemma}
\label{lemma:contlawbarycentermap}
Let $(\rho_m)_n,\rho $ in $\PP_2(\R^d)$ be such that $W_2(\rho_m,\rho)\rightarrow 0$.  We have: 
\begin{itemize}
\item[i)] For each $\nu \in \PP_2(\R^d)$ the sequence of laws $((\id,  S_{\rho_m}^{\nu} ) \# \rho_m)_m$  converges   w.r.t. $W_2$ in  $ \PP_2(\R^{d}\times\R^{d})$  to $(\id,  S_{\rho}^{\nu} ) \# \rho$ . 
\item[ii)] There exists in some probability space $(\Omega,{\cal F}, \P)$,  a sequence of r.v.  $(X_m)_m$ of laws $(\rho_m)_m $ and a r.v. $X$ of law $\rho$ such that, for each  $\nu \in \PP_2(\R^d)$, the sequence    $(X_m, S^{\nu}_{\rho_m}  (X_m))_m$  (with laws $((\id,  S_{\rho_m}^{\nu} ) \# \rho_m)_m$) converges   in $\L^2(\Omega,{\cal F}, \P)$ to  $(X, S^{\nu}_{\rho}  (X))$  (with law $(\id,  S_{\rho}^{\nu} ) \# \rho$). 
\end{itemize}
\end{lemma}
\begin{proof}[Proof of Lemma \ref{lemma:contlawbarycentermap}]
For the entire proof, we fix a $\nu\in \PP_2(\R^d)$ and we write $S_m:= S_{\rho_m}^{\nu}$  and $S:= S_{\rho}^{\nu}$ for simplicity. 

i)  By Theorem \ref{th:1.4Julio} and Part 1. of Theorem 1.5 in \cite{backhoff2020weak},  $ (S_m \# \rho_m)_m$ converges to  $ S \# \rho$ w.r.t. $W_1$ and, by Lemma \ref{lemma:UI2moments}, also with respect to $W_2$.   In particular,  the sequence    $((\id,  S_m ) \# \rho_m)_m$  has tight marginals, and therefore it is tight too.

 Let us identify its weak limiting points.  For simplicity we rename  $((\id,  S_m  ) \# \rho_m)_m$  a weakly convergent subsequence.  By the previous discussion, its weak limit $\d \hat{\rho}(x, z) $ clearly has first and second marginal laws equal to $\d\rho( x)$ and  $\d S\# \rho (z)$ respectively.   Moreover, $\int \| x\|^2 \d \rho_m(x)  + \int \|z\|^2 \d S_m\# \rho_m (z) \to  \int \| x\|^2+ \|z\|^2 \d \hat{\rho}(x,z) $, hence  $((\id,  S_m  ) \# \rho_m)_m$ converges to some $\hat{\pi}$ with respect to $W_2$ in $ \PP_2(\R^{d}\times\R^{d})$.  
 
 Now, by the characterisation of optimisers in  Theorem \ref{th:1.4Julio},   we have $V(\rho_m \vert \nu)=W_2^2(\rho_m, S_m \# \rho_m)=  \int \| x- S_m(x)\|^2 \d \rho_m(x) $.  Taking $m\to \infty$, and thanks to Theorem \ref{th:continuityV}, we finally obtain
  \[    V(\rho \vert\nu)=W_2^2(\rho, S\# \rho)=  \int \| x - z\|^2 \d \hat{\pi}(x,z).  \]
  In particular, using again  Theorem \ref{th:1.4Julio},  we conclude that  $ \d \hat{\pi}(x,z)  $ must be of the form $(\id, S)  \# \rho$.

ii) By Skorohod's representation theorem,  one can  construct  simultaneously in some probability space $(\Omega,{\cal F}, \P)$,  a sequence of r.v. $(X_m)_m$ of laws $(\rho_m)_m $ and a r.v. $X$ of law $\rho$ such that $(X_m)_m$ converges  $\P-$ a.s. to  $X$.  Moreover, since the sequence  $(\rho_m)_m$ converges  w.r.t. $W_2$ in  $ \PP_2(\R^{d})$, it has uniformly integrable second order moments. It follows that the sequence of  r.v. $( |X_m|^2)_m$ is uniformly integrable and,  by the Vitali convergence theorem,  that 
  $(X_m)_m$ also converges  to $X$  in $\L^2(\Omega,{\cal F}, \P)$. 
 
 Now,  by Lemma \ref{lemma:UI2moments}, the sequence of  r.v. $ (| S_m  (X_m)|^2)_n$  is uniformly integrable too. Thus, by the Vitali convergence theorem,  the statement will follow by proving that  $S_m  (X_m)$ converges  in $\P-$probability to $S (X)$.

 For each $N\in \N$, let $y\mapsto (y)^N$ denote the truncation of a vector $y\in \R^d$ obtained by projecting it onto the  centered ball of radius $N$, $(y)^N:= (1\wedge \frac{N}{|y|})y$, which is a  $1-$Lipschitz function bounded by $N$.   By Theorem \ref{th:1.4Julio}, the functions  $S^N_m:= (S_m)^N $  are then  $1-$Lipschitz and bounded uniformly   in $m\in \N$. Therefore,  by the Arzela-Ascoli theorem, their  restrictions to each compact cylinder set $R$ of $\R^d $ defines a relatively compact set of functions,  with respect to the uniform topology  in $C(R,\R^d)$.  It follows by a diagonal argument that some subsequence $(S_{m_k}^N)_k$ converges, uniformly on compact sets, to some continuous function $\tilde{S}$ on $\R^d$.  Since $X_n$ converges a.s. to the finite value $X$, we deduce that  $\P-$a.s. as $k\to \infty$, 
 \[ (X_{m_k} ,  S_{m_k}^N(X_{m_k}))\to (X,\tilde{S}(X)). \]
 Notice now that $ (X_{m_k} ,  S_{m_k}^N(X_{m_k}))$ has the law $(\id, (\cdot)^N\circ S_{m_k} ) \# \rho_{m_k}$ for each $k$ and thus, by part a) and continuity of the mapping $(x,y)\mapsto (x,(y)^N)$, the r.v.   $ (X,\tilde{S}(X))$ has the law $(\id, (\cdot)^N\circ S) \# \rho$. Hence we deduce that
  $$  (X,\tilde{S}(X)) =  (X, (S (X))^N)$$  
 $\P-$ almost surely.  The previous arguments can be applied not just to $(X_m)_m$ but to any subsequence of it. That is, we can similarly prove that any subsequence  of $(X_m, (S_m (X_m))^N)_m$  has  a subsequence that a.s. converges to $ (X, (S (X))^N)$. This means  that, for each $N\in \N$ 
 \[(X_m,  (S_m (X_m))^N)\to (X,  (S (X))^N) \]
 in $\P-$ probability when $n\to \infty$.  To conclude, by tightness we can find for each $\eta>0$ some $N\in \N$ large enough so that $\P( | S(X)|\geq N ) \leq \eta$ and $\P( | S_m(X_m)|\geq N )\leq \eta$ for all $m\in \N$, which yields for each $\varepsilon>0$, 
  \[\P( | S_m(X_m)- S(X)|\geq \varepsilon) \leq 2\eta + \P( | (S_m(X_m))^N- (S(X))^N|\geq \varepsilon).  \]
 Thus  $\limsup_m  \P( | S_m(X_m)- S(X)|\geq \varepsilon) \leq 2\eta  $ for arbitrary $\eta>0$ or, equivalently,  $\P( | S_m(X_m)- S(X)|\geq \varepsilon) \to 0$ as $m\to \infty$, which concludes the proof of b).
\end{proof}

We can now proceed to the proof of continuity of $G:\mu\mapsto \left(\sum_{i=1}^n \lambda_i S_{\mu}^{\nu_i}\right)\# \mu$.
\begin{proof}[Proof of Theorem \ref{th:continuityG}] 
Let  $(\rho_m)_m,\rho $ in $\PP_2(\R^d)$ such that $W_2(\rho_m,\rho)\rightarrow 0$. We need to  prove that $W_2^2(G(\rho_m),G(\rho))\rightarrow 0$. For each $m$, we write  $S_m^i:= S_{\rho_m}^{\nu_i}$ and $S^i:= S_{\rho}^{\nu_i}$.

By Lemma \ref{lemma:contlawbarycentermap}.ii),  there exists in some probability space  a sequence  $(X_m)_m$ of laws $(\rho_m)_m $ and a r.v. $X$ of law $\rho$ such that
\[(S_m^1(X_m),..., S_m^n(X_m) )\to (S^1(X),... ,S^n(X) )\quad \mbox{in} \quad\L^2(\P).\]
Therefore,   $\sum_{i=1}^n \lambda_i S_m^i(X_m)$ converges to $\sum_{i=1}^n\lambda_i S^i(X)$ in $\L^2(\P)$.  Since $\sum_{i=1}^n \lambda_i S_m^i(X_m)$ has law $G(\rho_m)$ and  $\sum_{i=1}^n\lambda_i S^i(X)$ has law $G(\rho)$, the proof is complete. 
\end{proof}

\begin{proof}[Proof of Proposition \ref{prop:fixedpoint}]
 As in \cite{alvarez2016fixed}, we easily see  that
\[\sum_{i=1}^n\lambda_i \int\Vert x-S_{\mu}^{\nu_i}(x)\Vert^2\d\mu(x) = \sum_{i=1}^n\lambda_i \int\Vert \bar{S}(x)-S_{\mu}^{\nu_i}(x)\Vert^2\d\mu(x) + \int\Vert x-\bar{S}(x)\Vert^2\d\mu(x).\]
But $\int\Vert x-S_{\mu}^{\nu_i}(x)\Vert^2\d\mu(x) = W_2^2(\mu,S^{\nu_i}_{\mu}\#\mu)$ since from Thm 1.4 in \cite{backhoff2019existence}  the barycentric map $S_{\mu}^{\nu_i}$ is an optimal map for the Monge problem between $\mu$ and $S_{\mu}^{\nu_i}\#\mu$. Moreover, by definition $G(\mu) = \bar{S}\#\mu$, therefore $\int\Vert x-\bar{S}(x)\Vert^2d\mu(x)\geq W_2^2(\mu,G(\mu))$. Finally, since $S_{\mu}^{\nu_i}\#\mu\leq_c \nu_i$,  we have that    $\int\Vert \bar{S}(x)-S_{\mu}^{\nu_i}(x)\Vert^2d\mu(x)\geq V(G(\mu)\vert \nu_i)$. This , recalling that $V(\mu | \nu_i) = W_2^2(\mu,S_{\mu}^{\nu_i}\#\mu)$,  yields
\begin{equation}
\label{ineq:fixedpoint}
 \sum_{i=1}^n\lambda_i V(\mu | \nu_i)  \geq \sum_{i=1}^n\lambda_i V(G(\mu)\vert \nu_i) + W_2^2(\mu,G(\mu)).
 \end{equation}
Therefore, if $\mu$ is a weak barycenter, we readily get that $\mu = G(\mu)$. 
\end{proof}

\begin{proof}[Proof of Proposition \ref{prop:convergence_finite}]
As in the proof of Theorem \ref{th:continuityG}, we denote $S_k^i$ an optimal barycentric projection associated to $\pi^{k,i}\in\Pi(\mu_k,\nu_i)$. First, we easily have that $\mu_{k+1}\in\PP_2(\R^d)$, indeed by Jensen's inequality
\[\int \Vert x \Vert^2\d\mu_{k+1}(x)= \iint \Vert \sum_{i=1}^n \lambda_i S_k^i(x)\Vert^2\d\mu_k(x) \leq  \sum_{i=1}^n \lambda_i \int\Vert y\Vert^2\d\nu_i(y) < \infty.\]
Then $(\mu_k)_k$ is tight,  with uniformly integrable $2$-moments by Lemma \ref{lemma:UI2moments}. Therefore $(\mu_k)_k$ admits a convergent subsequence in $W_2$. Let $\tilde{\mu}$ be a weak limit of a subsequence $(\mu_{k_j})_j$, then we have $W_2(\mu_{k_j},\tilde{\mu})\xrightarrow[j\to\infty]{} 0$. By continuity of $G$ in Theorem \ref{th:continuityG},  we get $W_2(\mu_{k_j+1},G(\tilde{\mu})) \xrightarrow[j\to\infty]{} 0$. Moreover, by Theorem \ref{th:continuityV} we have for $F(\mu):= \sum_{i=1}^n \lambda_i V(\mu | \nu_i)$ that $F(\mu_{k_j})\to F(\tilde{\mu})$ and $F(\mu_{k_j+1})\to F(G(\tilde{\mu}))$ as $j\to \infty$ . Let us prove that these two limits coincide.  From \eqref{ineq:fixedpoint}, we have
\[F(\mu_{k_j}) \geq \sum_{i=1}^n\lambda_i V(G(\mu_{k_j}) | \nu_i) = \sum_{i=1}^n\lambda_i V(\mu_{k_j+1}| \nu_i) = F(\mu_{k_j+1}).
\]
Iterating this inequality leads to   $F (\mu_{k_j})  \geq F(\mu_{k_j+1}) \geq F(\mu_{k_{j+1}})$  which yields   $F(\tilde{\mu}) = F(G(\tilde{\mu}))$   and then  $\tilde{\mu}=G(\tilde{\mu})$,  using inequality  \eqref{ineq:fixedpoint}.  Thus $(\mu_{k_j})_j$ converges w.r.t. $W_2$ to a probability distribution $\tilde{\mu}$ which is a fixed point of $G$.

\end{proof}

\begin{proof}[Proof of Lemma \ref{lemma:fixed_point}]
The proof is similar to that of \cite[Lemma 3.8]{rios1805bayesian}. For the sake of clarity, we rewrite it in our setting. We assume that $x=\int S_{\bar{\mu}}^{\nu}(x)\d\Q(\nu), \bar{\mu}(x)$-a.s. is not true, then
\begin{align*}
0 &< \int \Vert x-\int S_{\bar{\mu}}^{\nu}(x)\d\Q(\nu)\Vert^2\d\bar{\mu}(x)\\
& = \int\Vert x\Vert^2\d\bar{\mu}(x)-2\iint \langle x, S_{\bar{\mu}}^{\nu}(x)\rangle\d\Q(\nu)\d\bar{\mu}(x)+\int\Vert \int S_{\bar{\mu}}^{\nu}(x)\d\Q(\nu)\Vert^2\d\bar{\mu}(x).
\end{align*}
Moreover, $S_{\bar{\mu}}^{\mu}\#\bar{\mu} \leq_c \mu$, therefore by Theorem 1.4 in \cite{backhoff2019existence}, we get
\begin{align*}
\int V\left( \left[\int S_{\bar{\mu}}^{\nu}\d\Q(\nu)\right]\#\bar{\mu} | \mu\right)\d\Q(\mu)
&\leq \int \Vert \int S_{\bar{\mu}}^{\nu}\d\Q(\nu) - S_{\bar{\mu}}^{\mu}\Vert_{\L^2(\bar{\mu})}^2\d\Q(\mu)\\
& = \iint \Vert S_{\bar{\mu}}^{\nu}(x)\Vert^2\d\bar{\mu}(x)\d\Q(\nu)-\int\Vert \int S_{\bar{\mu}}^{\nu}\d\Q(\nu)\Vert^2\d\bar{\mu}(x).
\end{align*}
Finally, noticing that $\iint \Vert x-S_{\bar{\mu}}^{\nu}(x)\Vert^2\d\bar{\mu}(x)\d\Q(\nu) =\int V(\bar{\mu} | \nu)\d\Q(\nu)$, we hence get
\[\int V\left( \left[\int S_{\bar{\mu}}^{\nu}\d\Q(\nu)\right]\#\bar{\mu} | \mu\right)\d\Q(\mu) < \int V(\bar{\mu} | \nu)\d\Q(\nu),\]
which is in contradiction with $\bar{\mu}$ weak barycenter of $\Q$.
\end{proof}

In order to study the convergence of the iterative scheme in \eqref{eq:iterative_algo}, we define the following objects:
\begin{align}
L(\mu) &:= \frac{1}{2}\int V(\mu | \nu) \d\Q(\nu)\label{def:L}\\
H(\mu)(x) &:= -\int(S_{\mu}^{\nu}-\id)\d\Q(\nu)(x) \quad \, x\in \R^d .\label{def:H}
\end{align} 
Moreover, we denote by $\{\FF_k\}_k$ the filtration of the i.i.d. sample $\nu^k\sim\Q$, namely $\FF_{-1}$ is the trivial sigma-algebra and $\FF_{k+1}$ is the sigma-algebra generated by $\nu^0,\ldots,\nu^k$ and therefore $\mu_k$ in \eqref{eq:iterative_algo} is $\FF_k$-measurable.

The next two Propositions are needed to prove Theorem \ref{th:conv_iter}.

\begin{proposition}
\label{th:lsc}
The functions $\mu\in\PP_2(\R^d) \mapsto \Vert H(\mu)\Vert^2_{\L^2(\mu)}$  and $\mu\in\PP_2(\R^d) \mapsto L(\mu)$ are continuous w.r.t $W_2$.
\end{proposition}
\begin{proof}
Let  us first assume that $(\rho_m)_m,\rho $ in $\PP_2(\R^d)$ are such that $W_2(\rho_m,\rho)\rightarrow 0$. We want  to  prove that 
\begin{equation}\label{convergenceH}
\Vert H(\rho_m)\Vert^2_{\L^2(\rho_m)} \rightarrow \Vert H(\rho)\Vert^2_{\L^2(\rho)}
\end{equation}
when $m\to \infty$.    Consider the  probability space $(\Omega,\FF, \P)$ and r.v.'s   $(X_m)_m$ and  $X$ constructed in   Lemma  \ref{lemma:contlawbarycentermap}.ii), and recall that, for each $\nu \in \PP_2(\R^d)$,  the r.v.'s   $(X_m, S^{\nu}_{\rho_m}  (X_m))$ have law $(\id,  S_{\rho_m}^{\nu} )\# \rho_m$  for each $m$ and converge in  $\L^2(\Omega,\FF, \P)$ to  the r.v. $(X, S^{\nu}_{\rho}  (X))$, which  has the law $(\id,  S_{\rho}^{\nu} ) \# \rho$.
We  next extend this  construction in order to suitably randomise $\nu$.    More precisely, we enlarge the probability space $(\Omega,\FF, \P)$ to   the product space $(\bar{\Omega},\bar{\FF}, \bar{\P})= ( \Omega \times \PP_2(\R^d) , \FF\otimes \BB(\PP_2(\R^d)) , \P\otimes \Q)$, that is, we add an independent random variable, called $\bfnu$,  taking   values in   $ \PP_2(\R^d)$ and which has  distribution  $\Q$.  

Thanks to the measurability of the mappings  $(x,\nu)\mapsto S^{\nu}_{\rho_m}$  and  $(x,\nu)\mapsto S^{\nu}_{\rho}$ proven in Lemma   \ref{lemma:Smeasurable},  by  replacing $\nu$ by $\bfnu$ in the previous objects we obtain random  vectors   $ (X_m, S^{\bfnu}_{\rho_m}  (X_m) ) $ and $(X, S^{\bfnu}_{\rho}  (X))$ defined in  $(\bar{\Omega},\bar{\FF}, \bar{\P})$ which  have, conditionally on $\{\bfnu=\nu\} $,  the laws  $(\id,  S_{\rho_m}^{\nu}  ) \# \rho_m$ and  $(\id,  S_{\rho}^{\nu}  ) \# \rho$ respectively.   
Moreover, $\bfnu$ is independent of  the r.v. $X,X_1,\dots X_m$ under $\bar{\P}$. 

Now,  by conditioning on $\{\bfnu=\nu\} $, using  the convergence result in  Lemma  \ref{lemma:contlawbarycentermap}.ii)  and the dominated convergence Theorem,  we can easily check that  $ ((X_m, S^{\bfnu}_{\rho_m}  (X_m) ))_m $  converges to  $(X, S^{\bfnu}_{\rho}  (X))$   in $ \bar{\P}-$probability. 
 Furthermore, one can integrate w.r.t.  $\Q$ the bound \eqref{boundUI2mom} obtained for fixed $\nu$  in   the proof of  Lemma \ref{lemma:UI2moments} and, denoting $\bar{\E}$ the expectation with respect to $\bar{\P}$, deduce that
\begin{align*}
\sup_m  \bar{\E}   \left(   \| S^{\bfnu}_{\rho_m}  (X_m) \|^2 \mathbf{1}_{\{  \| S^{\bfnu}_{\rho_m}  (X_m) \|^2  \geq M\} }   \right)  \leq  \int & \int_{\{\Vert x\Vert^2\geq K\}}  \Vert x\Vert^2 \d \nu (x)  \Q(\d \nu)\\
&  + \frac{K}{M} \int \int  \Vert x\Vert^2 \d \nu (x)   \Q(\d \nu),
\end{align*}
for each $M,K\geq 0$, where the r.h.s. is finite since  $\Q\in\PP_2(\PP_2(\R^d))$.  It follows that the  sequence $( (X_m, S^{\bfnu }_{\rho_m}  (X_m) ))_m $ has uniformly integrable second moments, and therefore converges also in $L^2 (\bar{\Omega},\bar{\FF}, \bar{\P})$ to  $(X, S^{\bfnu}_{\rho}  (X))$, thanks to  the Vitali  convergence theorem. In particular, as $m$ tends to infty, we have
$$\int V(\mu_m | \nu) \d\Q(\nu) = \E \| X_m -S^{\bfnu}_{\rho_m}  (X_m) \|^2 \rightarrow  \E \| X-  S^{\bfnu}_{\rho}  (X) \|^2 = \int V(\mu | \nu) \d\Q(\nu),$$
which proves the continuity of the function $\mu\in\PP_2(\R^d) \mapsto L(\mu)$.
 
We observe now that   $ \bar{\E} \left(  S^{\bfnu }_{\rho_m}  (X_m)  \vert X_m \right) = \int S_{\rho_m}^{\nu}(X_m)\d\Q(\nu) $  and $\bar{\E} \left(  S^{\bfnu}_{\rho}  (X)  \vert X \right) = \int S_{\rho}^{\nu}(X)\d\Q(\nu) $, $ \bar{\P}-$ a.s., Moreover,  if $\mathcal{F_{\infty}}$ denotes the $\sigma$-algebra generated by $(X_1,X_2,\ldots)$, one  has   $ \bar{\E} \left(  S^{\bfnu }_{\rho_m}  (X_m)  \vert \mathcal{F_{\infty}}  \right) =   \bar{\E} \left(  S^{\bfnu }_{\rho_m}  (X_m)  \vert X_m \right)$    and     $  \bar{\E} \left(  S^{\bfnu }_{\rho}  (X)  \vert \mathcal{F_{\infty}}  \right) = \bar{\E} \left(  S^{\bfnu }_{\rho}  (X)  \vert X \right)$.    
Using the continuity in $L^2 (\bar{\Omega},\bar{\FF}, \bar{\P})$ of the conditional expectation with respect to  $\mathcal{F_{\infty}}$, we deduce that 
\begin{equation}
\label{eq:convXSL2}
  X_m-  \bar{\E} \left(  S^{\bfnu }_{\rho_m}  (X_m)  \vert X_m \right) \to  X-  \bar{\E} \left(  S^{\bfnu }_{\rho}  (X)  \vert X \right)
  \end{equation}
in $L^2 (\bar{\Omega},\bar{\FF},  \bar{\P})$. 
We conclude that  $\bar{\E}\Vert X_m- \bar{\E}(S_{\rho_m}^{\bfnu}(X_m)|X_m)\Vert^2 \to  \bar{\E}\Vert X- \bar{\E}S_{\rho}^{\bfnu}(X)|X)\Vert^2$ as $m\to \infty$, which is exactly   the required convergence \eqref{convergenceH}. 

\end{proof}

\begin{proposition}
\label{prop:ineqsto}
For the sequence $(\mu_k)_k$ defined in \eqref{eq:iterative_algo}, we have
\begin{equation}
\E(L(\mu_{k+1})-L(\mu_k) | \FF_k) \leq \gamma_k^2L(\mu_k)-\gamma_k\Vert H(\mu_k)\Vert^2_{\L^2(\mu_k)}.
\end{equation}
\end{proposition}
\begin{proof}
The arguments are similar to the ones used for the population Wasserstein barycenter iterative scheme in the proof of Proposition 3.2 in \cite{bakchoff2022}. Let us set them for the present problem. Let $\nu\in\textrm{supp}(Q)$, then $([(1-\gamma_k)\id+\gamma_k S_{\mu_k}^{\nu^k}],S_{\mu_k}^{\nu}]\#\mu_k$ belongs to $\Pi(\mu_{k+1},S_{\mu_k}^{\nu}\#\mu_k)$. Therefore we have
\begin{align*}
V(\mu_{k+1} | \nu) &\leq W_2^2(\mu_{k+1}, S_{\mu_k}^{\nu}\#\mu_k)\quad \mbox{since} \ S_{\mu_k}^{\nu}\#\mu_k\leq_c \nu\\
& \leq \int \Vert (1-\gamma_k)x+\gamma_k S_{\mu_k}^{\nu^k}(x)-S_{\mu_k}^{\nu}(x)\Vert^2\d\mu_k(x)\\
&= \int \Vert x-S_{\mu_k}^{\nu}(x)\Vert^2\d\mu_k(x)-2\gamma_k\int\langle  x-S_{\mu_k}^{\nu}(x), x-S_{\mu_k}^{\nu^k}(x)\rangle\d\mu_k(x)\\
& \qquad +\gamma_k^2\int \Vert x-S_{\mu_k}^{\nu^k}(x)\Vert^2\d\mu_k(x)\\
& = V(\mu_k | \nu) + \gamma_k^2V(\mu_k | \nu^k) -2\gamma_k\int\langle  x-S_{\mu_k}^{\nu}(x), x-S_{\mu_k}^{\nu^k}(x)\rangle\d\mu_k(x).
\end{align*}
Integrating with respect to $\nu$, and divided by $2$ we get
$$L(\mu_{k+1})\leq L(\mu_k)+\frac{\gamma_k^2}{2} V(\mu_k | \nu^k)-\gamma_k\int\langle  H(\mu_k)(x), x-S_{\mu_k}^{\nu^k}(x)\rangle\d\mu_k(x).$$
We can then take the conditional expectation with respect to the filtration $\FF_k$, knowing that $\mu_k$ is $\FF_k$-measurable and that $\nu^k$ is independently sampled from $\FF_k$, we have
\begin{align*}
\E(L(\mu_{k+1}) | \FF_k)&\leq L(\mu_k) +\frac{\gamma_k^2}{2} \int V(\mu_k | \nu)\d\Q(\nu)- \gamma_k\int\langle  H(\mu_k)(x), \int x-S_{\mu_k}^{\nu}(x)\d\Q(\nu)\rangle\d\mu_k(x)\\
&= L(\mu_k) +\gamma_k^2 L(\mu_k)- \gamma_k\int\langle  H(\mu_k)(x), H(\mu_k)(x)\rangle\d\mu_k(x)\\
&=(1+\gamma_k^2) L(\mu_k) - \gamma_k \Vert H(\mu_k)\Vert^2_{\mu_k}.
\end{align*}
\end{proof}

\begin{proof}[Proof of Theorem \ref{th:conv_iter}]

We will proceed in a similar way as in the proof of Theorem 1.4 in \cite{bakchoff2022}. Let us first note that the set $K_{\Q}$ is compact in ${\cal P}_2(\R^d)$ w.r.t. $W_2$ (see \cite{villani2008optimal}). Moreover,  the sequence   $(\mu_k)_k$ is a.s. included in this compact set, as can be seen by induction using the facts that the function $|\cdot |^{2+\epsilon}$ is convex and that $S_{\mu_k}^{\nu_k}\# \mu_k \leq_c \nu_k$, with $\nu^k \sim\Q$.   Now let $\bar{\mu}$ be a weak population barycenter, i.e. $\bar{\mu}$ minimises $L$ defined in \eqref{def:L}, and write $\bar{L}:=L(\bar{\mu}) $. We introduce the sequences
$$h_k:=L(\mu_k)-\bar{L}\geq 0 \quad \mbox{and} \quad \alpha_k:=\prod_{i=1}^{k-1}\frac{1}{1+\gamma_k^2}.$$
From condition \eqref{cond:gamma}, the sequence $(\alpha_k)_k$ converges to some $\alpha_{\infty}>0$. By Proposition \ref{prop:ineqsto}, we have
\begin{align}
&\E(h_{k+1}-(1+\gamma_k^2)h_k|\FF_k)\leq \gamma_k^2\bar{L} -\gamma_k\Vert H(\mu_k)\Vert^2_{\mu_k}\leq \gamma_k^2\bar{L}  \nonumber\\
\Rightarrow & \, \E(\alpha_{k+1}h_{k+1}-\alpha_kh_k | \FF_k)\leq  \alpha_{k+1}\gamma_k^2\bar{L}  \quad\mbox{upon multiplying by}\ \alpha_{k+1}.\label{eq:th:conv_iter_proof}
\end{align} 
Defining now
$$\delta_k :=\left\{
    \begin{array}{ll}
        1 & \mbox{if } \ \E(\alpha_{k+1}h_{k+1}-\alpha_kh_k | \FF_k)>0\\
        0 & \mbox{otherwise,}
    \end{array}
\right.   $$
we deduce that
\begin{align*}
\sum_{k=1}^{\infty} \E(\delta_k(\alpha_{k+1}h_{k+1}-\alpha_kh_k)) &= \sum_{k=1}^{\infty} \E(\delta_k \E(\alpha_{k+1}h_{k+1}-\alpha_kh_k| \FF_k))\\
&\leq L(\bar{\mu})\sum_{k=1}^{\infty} \alpha_{k+1}\gamma^2_k\leq L(\bar{\mu})\sum_{k=1}^{\infty} \gamma^2_k<\infty.
\end{align*}
Since $h_k\alpha_k\geq 0$, by the quasi-martingale convergence theorem $(h_k\alpha_k)_k$ converges almost surely. Since $(\alpha_k)_k$ converges to $\alpha_{\infty}>0$, then $(h_k)_k$ also converges almost surely to some $h_{\infty}\geq 0$. Taking expectations in Eq.~\eqref{eq:th:conv_iter_proof} and summing we get
$$\E(\alpha_{k+1}h_{k+1})\leq \E( \alpha_0h_0) + \bar{L}\sum_{m=1}^k\alpha_{m+1}\gamma_m^2\leq C.$$
Fatou's Lemma yields $ \E(\alpha_{\infty}h_{\infty})<\infty$, and since $\alpha_{\infty}<\infty$, we deduce that $h_{\infty}$ is almost surely finite. Our goal  now is to show that $h_{\infty}=0$.  
From \eqref{eq:th:conv_iter_proof} we deduce as before that 
$$\E (\alpha_{k+1}h_{k+1})-\E(\alpha_0h_0)\leq \bar{L} \sum_{m=1}^k\alpha_{m+1}\gamma_m^2- \sum_{m=1}^k\alpha_{m+1}\gamma_m\E\left(\Vert H(\mu_m)\Vert^2_{\L_2(\mu_m)}\right).$$
Taking liminf, using  Fatou on the l.h.s. and monotone convergence on the r.h.s. we obtain that
$$-\infty<\E(\alpha_{\infty}h_{\infty}) -\E(\alpha_0h_0)\leq C-\E\left(\sum_{m=1}^{\infty}\alpha_{m+1}\gamma_m\Vert H(\mu_m)\Vert^2_{\L_2(\mu_m)}\right)$$
hence, in particular,
\begin{equation}
\label{eq:sum_product}
\sum_{k=1}^{\infty}\gamma_k\Vert H(\mu_k)\Vert^2_{\L_2(\mu_k)}<+\infty \quad \mbox{a.s.}
\end{equation}

Note that the conditions in Eqs.~\eqref{eq:sum_product} and \eqref{cond:gamma} imply that
	\begin{equation}\label{eq:liminf0} \textstyle
	\liminf_{t\to \infty} \Vert H(\mu_k)\Vert^2_{\L_2(\mu_k)} =0,\,\, \text{a.s.}
\end{equation}
	Observe also that, by the compactness of $ K_\Q $ and the continuity of $L$ in  Proposition \ref{th:lsc}, the set $\{ \rho : L(\rho) \geq \bar{L}  +\delta  \} \cap K_\Q $  is  $W_2$-compact.  Therefore, the  function  $\rho\mapsto \lVert H(\rho)\rVert^2_{\L^2(\rho)}$,  also $W_2$-continuous by  Proposition \ref{th:lsc},   attains its minimum on that set.  That minimum cannot be zero, as otherwise we would have obtained a fixed-point that is not a weak barycenter, contradicting our hypothesis. It follows that
	\begin{equation}\label{eq:infepsK} \textstyle
	\forall \delta >0 , \, \inf_{\{ \rho : L(\rho) \geq \bar{L}  +\delta \} \cap K_\Q}   \lVert H(\rho)\rVert^2_{\L^2(\rho)}>0.  
	\end{equation}
	Since $\{ \mu_k\}_k\subset K_{\Q}$ a.s.,  we deduce from the previous result the  a.s.\ inclusions of  events:
	\begin{equation*}
	\begin{split}
	\{ h_{\infty} \geq  2\delta  \} \subset \, &  \textstyle\left\{ \mu_t \in   \{ \rho : L(\rho) \geq L(\bar{\mu}) +\delta  \} \cap K_\Q 
	\, \forall t \mbox{ large enough} \right\} \\
	  \subset \, & \textstyle \bigcup_{\ell\in \N}     \left\{  \lVert H(\mu_t)\rVert^2_{\L^2(\mu_t)} > 1/\ell :\, \forall  t \mbox{ large enough} \right\} 
	  \subset \, \textstyle \left\{ \liminf_{t\to \infty} \lVert H(\mu_t)\rVert^2_{\L^2(\mu_t)}>0\right\}  ,
	 \end{split}
	 \end{equation*}
	 where Eq.~\eqref{eq:infepsK} was used to obtain the second inclusion. It follows using  Eq.~\eqref{eq:liminf0}    that $\P(h_{\infty} \geq  2\delta )=0$ for every $\varepsilon>0$, hence $h_{\infty}=0$ a.s. as claimed. In other words,  $L(\mu_k)\to \bar{L}$  a.s. as $k\to \infty$.
	 
	 We already established that $\{ \mu_k\}_k\subset K_{\Q}$,  hence the sequence is relatively compact. We finally conclude that the limit $\hat{\mu}$ of any convergent subsequence $\{ \mu_{k_j} \}_{k_j}$ satisfies $L(\hat{\mu})= \lim_j L (\mu_{k_j}) = \bar{L} $, whence, it is a weak barycenter.

\begin{remark}
\label{remark:assumptionA}
Assumption (A) can be replaced by the following more general condition:

\begin{minipage}{1cm}
\vspace{-0.4cm}(A')
\end{minipage}
\begin{minipage}{12cm} $\Q$ gives full measure to a $W_2$-compact set $K_\Q$ which is  \emph{``weakly geodesically closed''}, in the sense that for any $\mu,\nu\in K_\Q$ and $t\in[0,1]$,  $((1-t)\textnormal{id}+tS_\mu^\nu)\#\mu\in K_\Q$. \end{minipage} 

\end{remark}

\end{proof}


\section{Numerical results}
\label{sec:more_expe}

\subsection{Proximal algorithm for the computation of the OWT plan}
\label{sec:appendix_proximal}

This section is dedicated to the resolution of the OWT problem. Let $\mu=\sum_{i=1}^r a_i\delta_{x_i}$ and $\nu=\sum_{j=1}^m b_i\delta_{y_j}$, be two discrete measures, the OWT problem boils down to solving
\begin{equation}
\label{eq:solve_S}
\min_{\pi\in\R^{r\times m}} \underbrace{\sum_{i=1}^r a_i  \Vert x_i- \left(\frac{\pi {\bf y}}{\bf a}\right)_i\Vert^2}_{f(\pi)} + \underbrace{\vphantom{\sum_i} 1_{\Pi(\mu,\nu)}(\pi)}_{g(\pi)},
\end{equation}
where $1_C$ is the indicator function of the set $C$ i.e.
$$
1_C(\pi) = \left\{
    \begin{array}{ll}
        \pi & \mbox{if } \pi \in C \\
        \infty & \mbox{otherwise.}
    \end{array}
\right.
$$
The proximal algorithm to solve Eq.~\eqref{eq:solve_S} then reads:
\begin{equation}
\label{eq:prox_algo}
\pi^{\ell+1} = {\rm prox}_{\theta^{\ell}g}(\pi^{\ell}-\theta^{\ell}\nabla f(\pi^{\ell})).
\end{equation}
As $\Pi(\mu,\nu)$ is a closed non-empty convex set, the proximal operator of $g$ reduces to the Euclidean projection onto $\Pi(\mu,\nu)$:
$${\rm proj}_{\Pi(\mu,\nu)}(P) = \argmin_{\pi\in\R^{r\times m}} \Vert P-\pi\Vert^2 = \argmin_{\pi\in\R^{r\times m}} \ \langle \pi,-P\rangle + \frac{1}{2}\Vert \pi\Vert^2$$
where $\Vert\cdot\Vert$ is the Frobenius norm. This projection problem can be solved by Dykstra’s algorithm with alternate Bregman projections \cite{dessein2016regularized} or by stochastic dual approaches of OT regularised by an $\L_2$ norm \cite{seguy2017large}. This method is summarised in Algorithm \ref{algo:bary_proj}. In particular, we used an accelerated version of Eq.~\eqref{eq:prox_algo} via FISTA \cite{beck2009fast} (with $\omega_{\ell} \in[0,1)$ an extrapolation parameter and $\theta_{\ell}$ the usual stepsize chosen by a line search) in order to compute an optimal plan $\pi_{\mu}^{\nu}$ in the weak transport problem. The optimal barycentric projection is then given by $S_{\mu}^{\nu}=\frac{\pi_{\mu}^{\nu}{\bf y}}{{\bf a}}$. We initialised the algorithm with a random matrix whose elements sum to $1$. Observe that, from Algorithm \ref{algo:weakbar}, the $K$ optimal barycentric projection computations can be parallelised for each step $n$.
\begin{algorithm}
\SetAlgoLined
{\bfseries Output:} $\pi_{\mu}^{\nu}$;\\
{\bfseries Input:} $\mu = \sum_{i=1}^r a_i\delta_{x_i}$ and $\nu = \sum_{j=1}^m b_j\delta_{y_j}$;\\
 Initialise $\pi_0$ random matrix;\\
 \While{not converge}{
$P_{{\ell}+1} := \pi_{\ell}+\omega_{\ell}(\pi_{\ell}-\pi_{{\ell}-1})$;\\
$\pi_{{\ell}+1} := {\rm proj}_{\Pi(\mu,\nu)}(P_{{\ell}+1}-\theta_{\ell}\nabla f(P_{{\ell}+1}))$;\\
 }
 \caption{Computation of the optimal weak plan}
 \label{algo:bary_proj}
\end{algorithm}

With respect to the efficiency of this algorithm, Figure \ref{fig:efficiency} shows a comparison of different settings for Eq.~\eqref{eq:prox_algo} in order to compute an optimal weak transport plan. For that purpose, we considered two discrete distributions $\mu$ and $\nu$ each constructed from $r=m=10, 100$ and $250$ samples of two dimensional Gaussian measures. We illustrate the convergence for both the standard and accelerated versions of the proximal algorithm, as well as for the projection into $\Pi(\mu,\nu)$ via Dykstra’s algorithm or the stochastic dual approach. As expected, the accelerated version of Eq.~\eqref{eq:prox_algo} converges faster than the classical proximal algorithm, and the projection step in more stable with Dykstra’s algorithm. Moreover, the smaller the number of support points, the faster the convergence. We have also noted that the random initialisation does not affect the convergence towards the minimiser of Eq.~\eqref{eq:solve_S}.
\begin{figure}
\hspace{-0.3cm}\begin{tabular}{C{4.3cm} C{4.3cm} C{4.3cm}}
\includegraphics[scale=0.29]{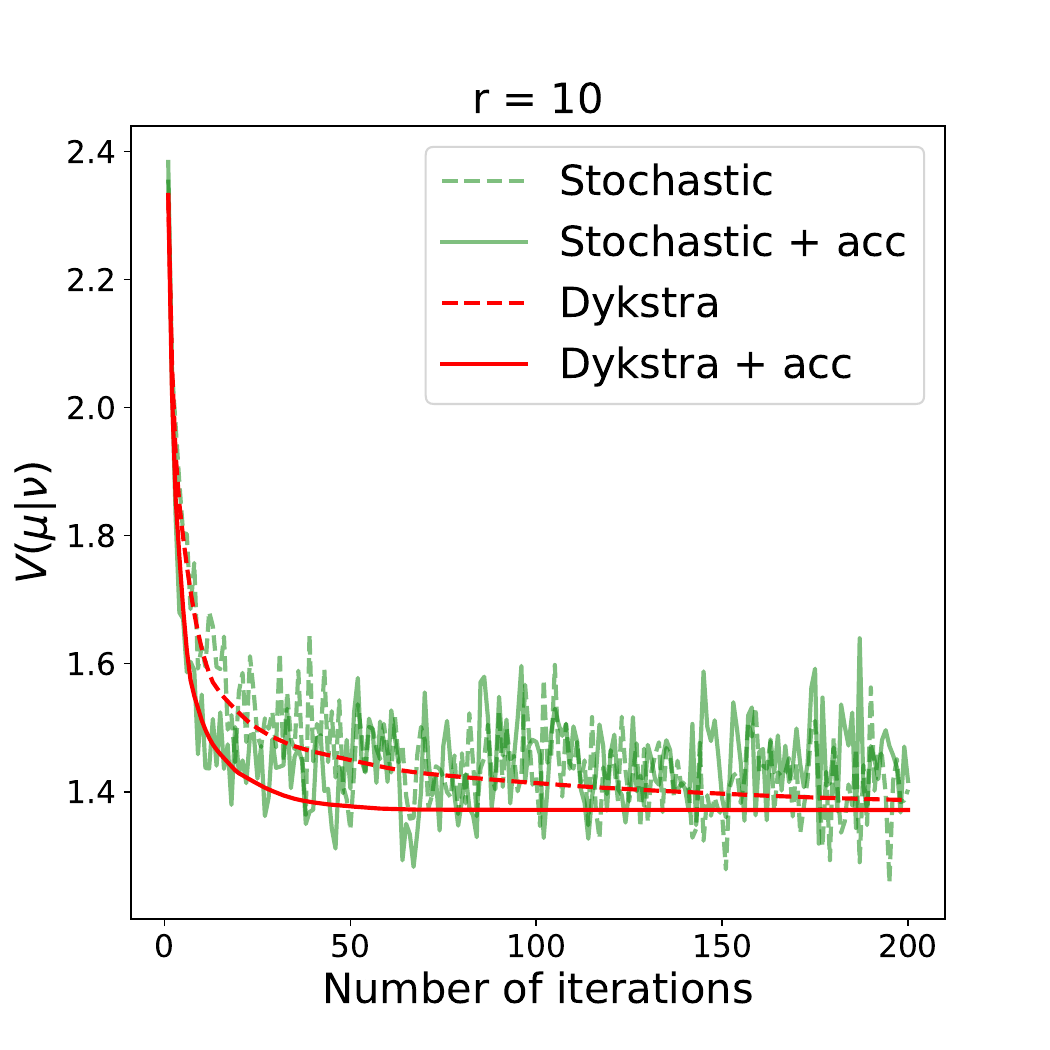} &
 \includegraphics[scale=0.29]{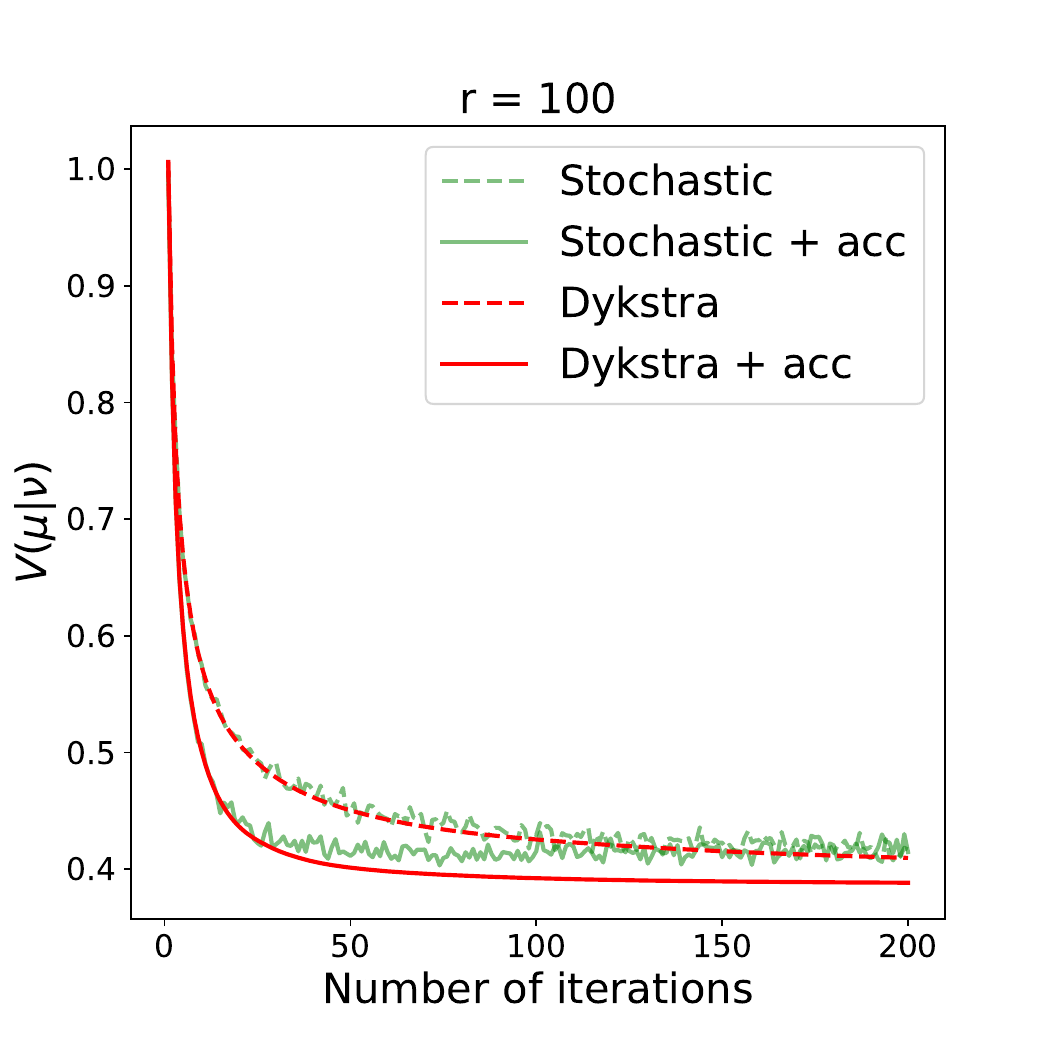} &
  \includegraphics[scale=0.29]{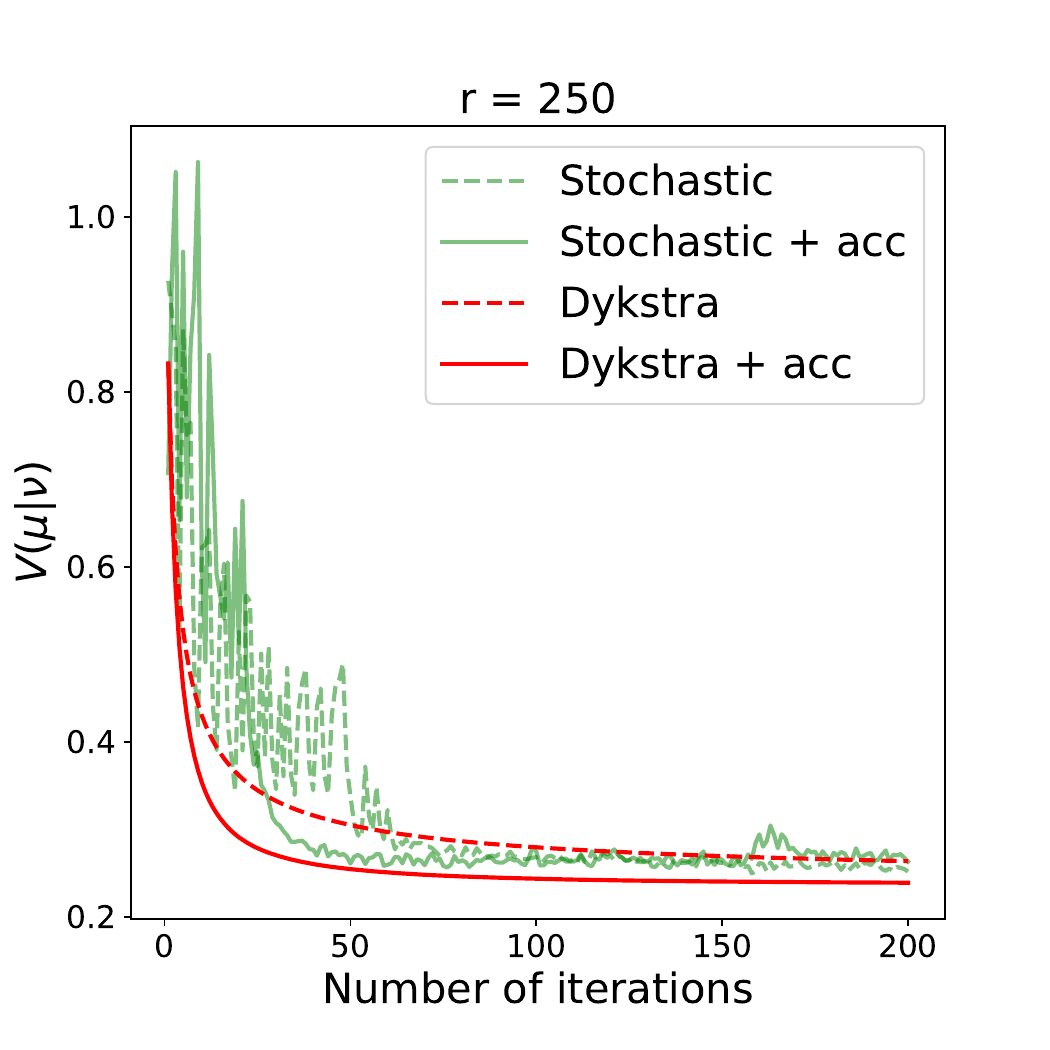} 
\end{tabular}
\caption{Convergence of the algorithm \eqref{algo:bary_proj} in several settings for measures $\mu$ and $\nu$ supported on $r=m$ points.}
\label{fig:efficiency}
\end{figure}

\subsection{Additional experiments}
\label{sec:additional_expe}

{\bf Gaussian distributions} As in Section 5 of \cite{alvarez2016fixed}, we computed a weak barycenter between two 2D centered ellipses $E(\Sigma_i)=\{s\in\R^2 : x^t\Sigma_i^{-1}x=1\}$ with covariances matrices
$$\Sigma_1=\begin{pmatrix} 2 & 0 \\ 0 & 1\end{pmatrix}\quad \mbox{and} \quad \Sigma_2=\begin{pmatrix} 1 & 0 \\ 0 & 2\end{pmatrix},$$
by considering $300$ random observations foe each ellipse. We then executed the iterations of Algorithm \ref{algo:weakbar} until the difference of the objective function (i.e., the sum in Eq.~\eqref{eq:weak_barycenter}) between two successive iterations was smaller than $1e-5$. This occurred at the 8th iteration, and the resulting weak barycenter was a circle within both ellipses. As we have access to the value of the weak barycenter problem (see Eq.~\eqref{eq:obj_function}), we also compared the value of the objective function at the 8th iteration (that is $3.62e-4$) to $\frac{1}{2}\sum_{i=1}^2  \Vert \E(Y_i)\Vert^2-\Vert \frac{1}{2} \sum_{i=1}^2  \E(Y_i)\Vert^2$, with a plug-in estimator for $\E(Y_i)$. The approximated objective was equal to $3.21e-4$, therefore, Algorithm \ref{algo:weakbar} gave a satisfactory optimised weak barycenter.

{\bf Ellipse distributions} ($r=100$ \& $K=15$). We considered ellipse distributions with random center in $(-5,5)$, random semi-major and semi-minor axes in $(6,14)$. The results are presented in Fig.~\ref{fig:rings}, where the same conclusions as in the Gaussian examples hold.
\begin{figure}[ht!]
\hspace{-0.4cm}\begin{tabular}{C{4.5cm} C{4.5cm} C{4.5cm}}
\includegraphics[scale=0.23]{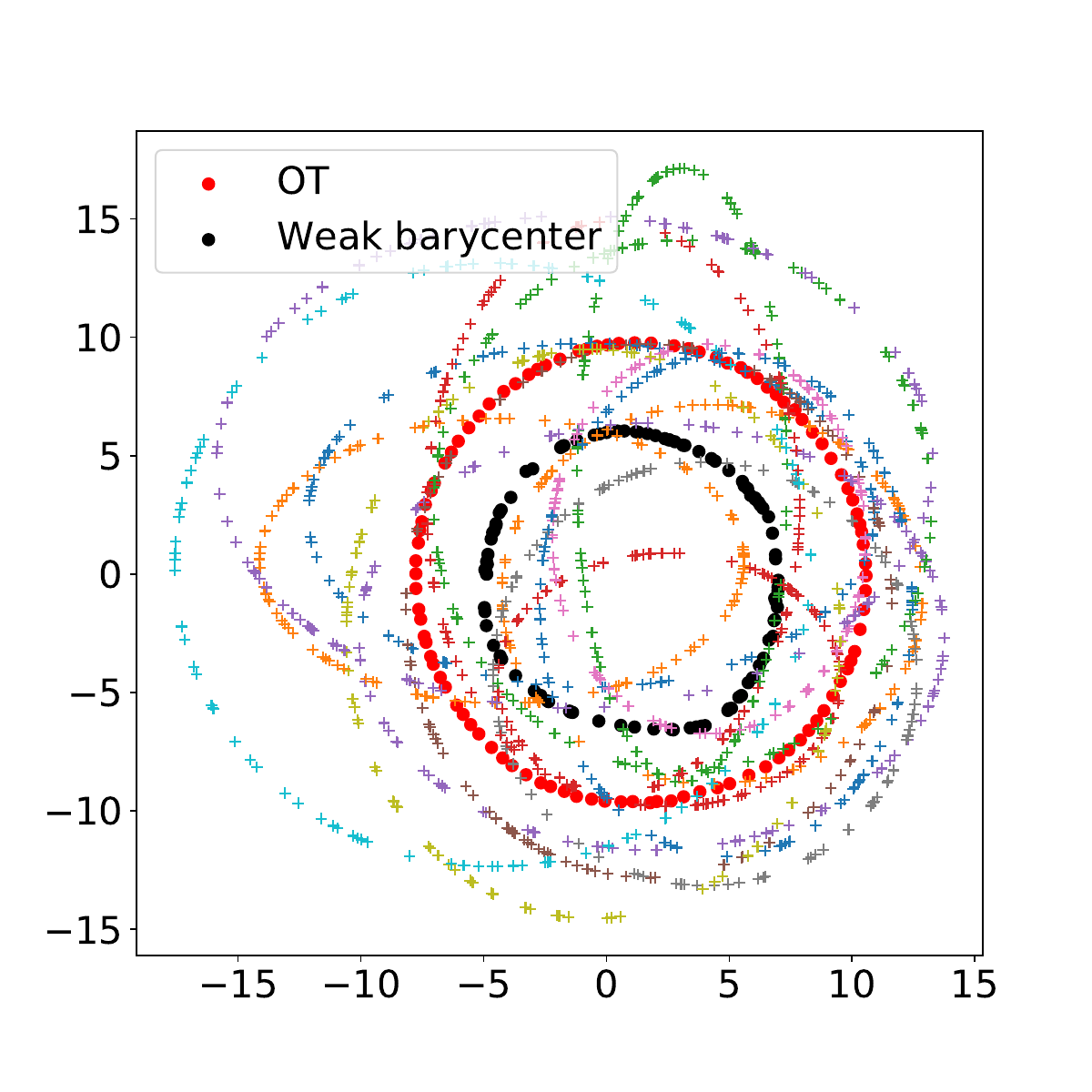} &
\includegraphics[scale=0.23]{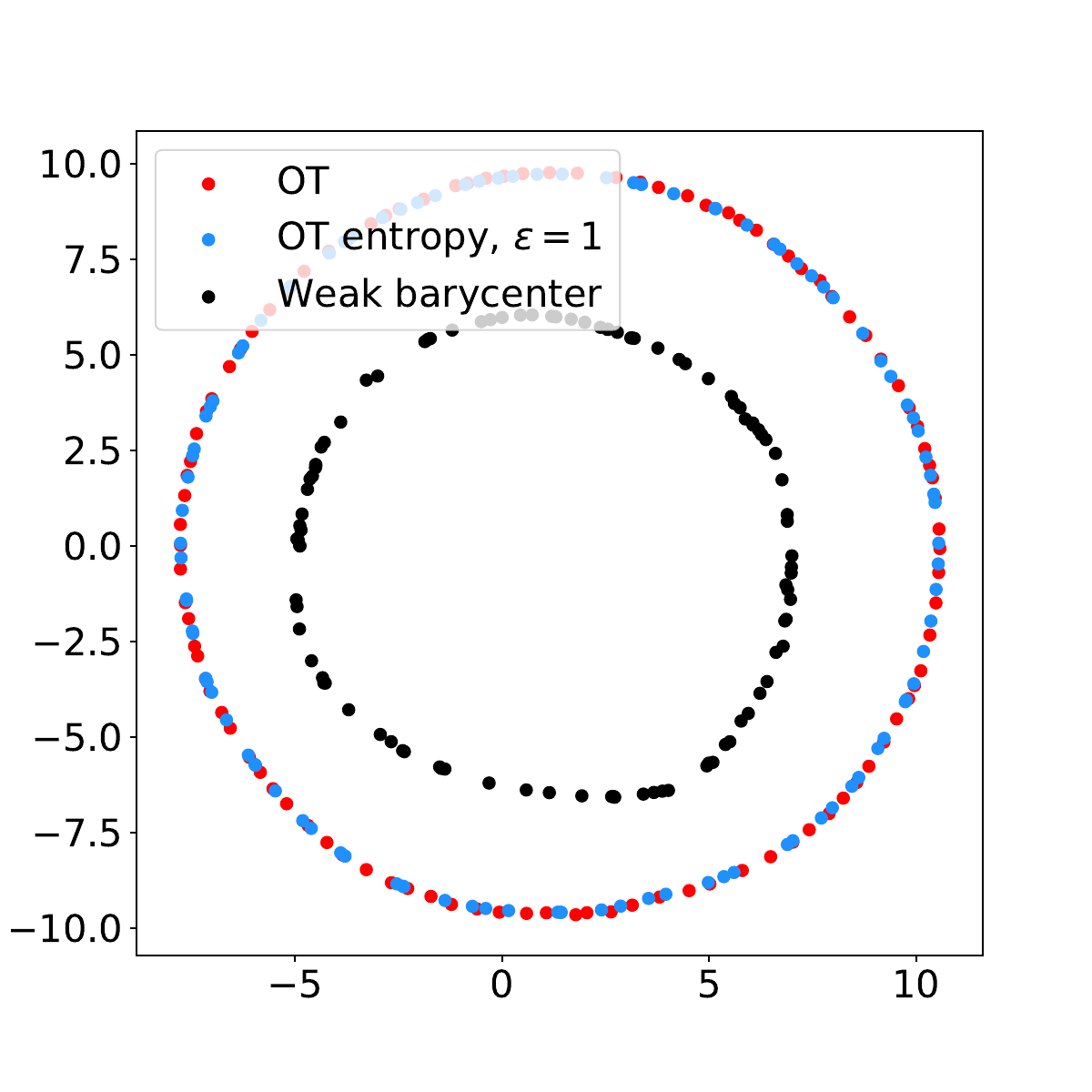} &
\includegraphics[scale=0.23]{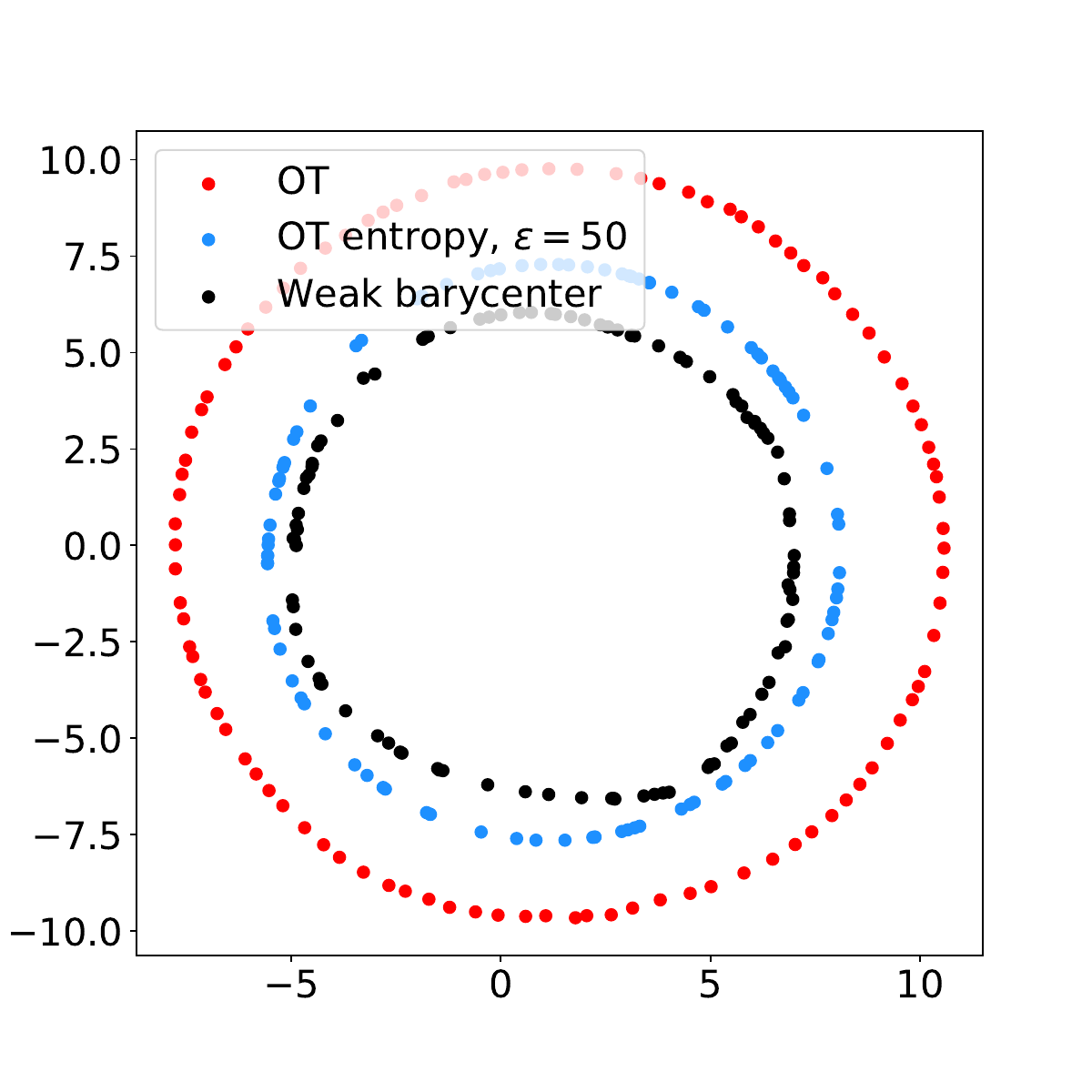} 
\end{tabular}
\caption{(left) Ellipse distributions and their OWT (black) and OT (red) barycenters computed with Algorithm \ref{algo:weakbarpop}. (center $\&$ right) Illustration of the weak (black), OT (red) and OT Sinkhorn (blue) barycenters for different values of $\varepsilon=1,50$.}
\label{fig:rings}
\end{figure}

{\bf Pair-of-ellipses} ($r\in (200,300)$ \& $K=10$). In the same fashion, we considered  distributions supported on two ellipses with random centers in $(-5,5)$, random semi-major and semi-minor axes in $(1,7)$ and $(7,13)$ respectively. Fig.~\ref{fig:pair_ellipses} shows the distributions (left) as well as the OT and OWT barycenters (right) computed from random samples of the distributions. Observe that, once again, the weak barycenter better preserved the structure of the distributions when computing Algorithm \ref{algo:weakbarpop}.

\begin{figure}[ht!]
\centering
\begin{tabular}{cc}
\includegraphics[scale=0.26]{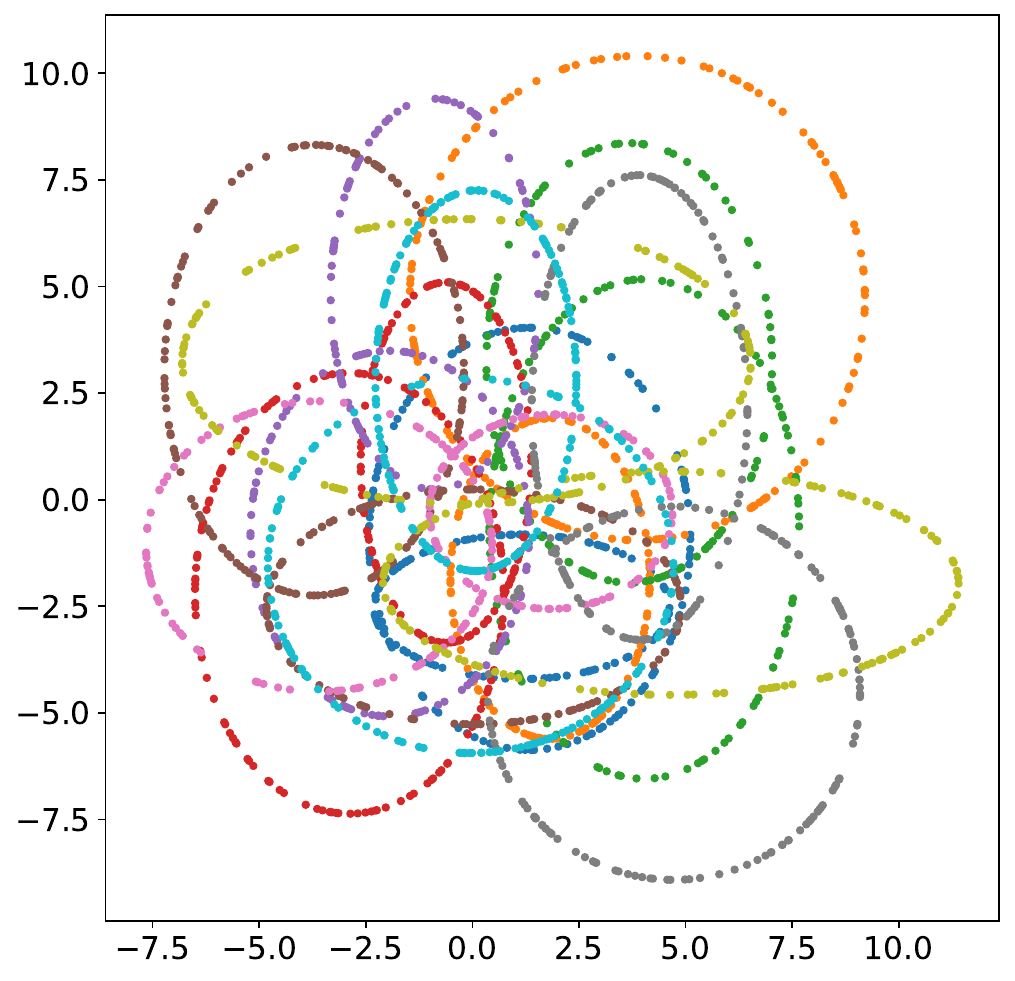} &
\includegraphics[scale=0.26]{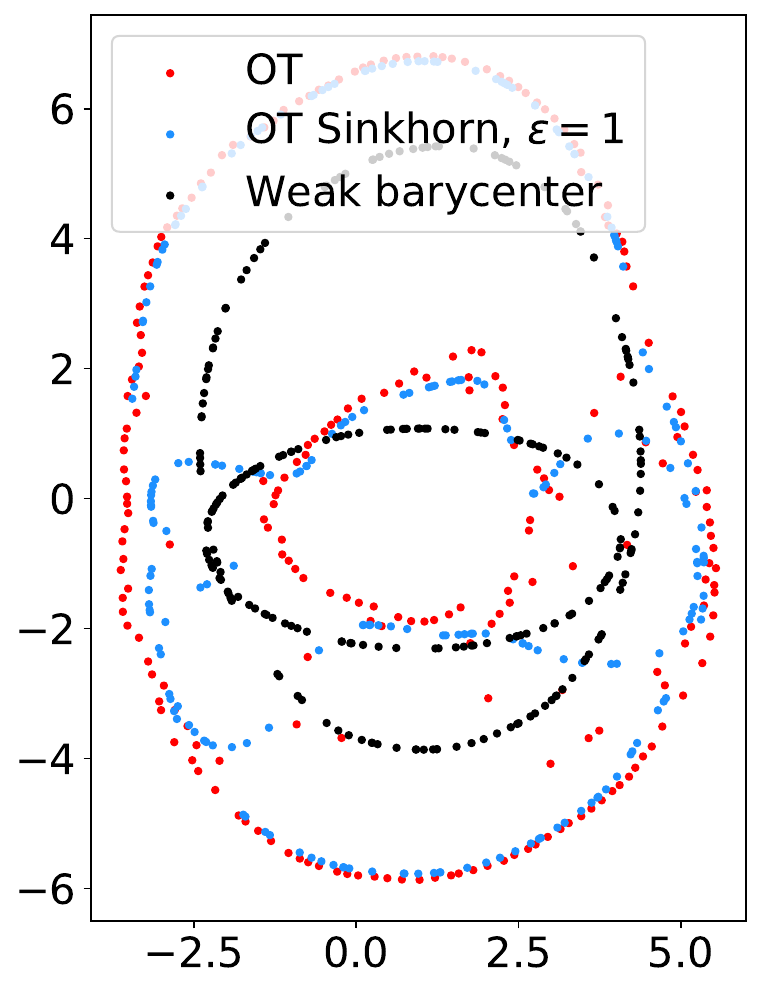} 
\end{tabular}
\caption{(left) Distributions supported on a pair-of-squares. (right) OWT (black), OT (red) and OT Sinkhorn for $\varepsilon=1$ (blue) barycenters  computed with Algorithm \ref{algo:weakbarpop}.}
\label{fig:pair_ellipses}
\end{figure}